\documentclass[journal]{IEEEtran}
\usepackage{amssymb}
\usepackage{amsmath}
\usepackage{dcolumn}
\usepackage{bm}
\usepackage{color}
\usepackage{graphicx}
\usepackage{subfigure} 
\usepackage{tabularx}
\usepackage{array}
\usepackage{amsthm}
\usepackage{cite}
\usepackage[ruled,linesnumbered]{algorithm2e}
\usepackage{tabu}                     
\usepackage{multirow}                 
\usepackage{multicol}                 
\usepackage{multirow}                
\usepackage{float}                    
\usepackage{makecell}                 
\usepackage{booktabs}
\usepackage{threeparttable}  
\usepackage{tikz}

\usepackage{etoolbox}
\makeatletter
\patchcmd{\@makecaption}
  {\scshape}
  {}
  {}
  {}
\makeatother

\SetKwInput{KwData}{Input}
\SetKwInput{KwResult}{Output}

\usepackage[colorlinks,
            linkcolor=red,
            anchorcolor=blue,
            citecolor=green
            ]{hyperref}

\newtheorem{theorem}{Theorem}
\newtheorem{corollary}{Corollary}
\newtheorem{lemma}{Lemma}
\newtheorem{remark}{Remark}
\newtheorem{definition}{Definition}
\newtheorem{assumption}{Assumption}
\newtheorem{proposition}{Proposition}
\newtheorem{problem}{Problem}
\newtheorem{example}{Example}

\renewcommand{\t}{^{\mbox{\tiny\sf T}}}
\newcommand{\bremark}{\begin{remark}
\begin{rm}}
\newcommand{\eremark}{ \end{rm}\hfill \rule{1mm}{2mm}
\end{remark} }
\newcommand{\btheorem}{\begin{theorem} \begin{it}}
\newcommand{\etheorem}{\end{it} \hfill \rule{1mm}{2mm}
\end{theorem} }
\newcommand{\blemma}{\begin{lemma} \begin{it} }
\newcommand{\elemma}{ \end{it} \hfill\rule{1mm}{2mm}
\end{lemma} }
\newcommand{\bcorollary}{\begin{corollary} \begin{it} }
\newcommand{\ecorollary}{ \end{it} \hfill\rule{1mm}{2mm}
\end{corollary} }
\newcommand{\bdefinition}{\begin{definition} }
\newcommand{\edefinition}{ \hfill\rule{1mm}{2mm}
\end{definition} }
\newcommand{\bproposition}{\begin{proposition} }
\newcommand{\eproposition}{\hfill \rule{1mm}{2mm}
\end{proposition} }
\newcommand{\bexample}{\begin{example} \begin{rm}}
\newcommand{\eexample}{ \end{rm} \hfill\rule{1mm}{2mm}
\end{example} }

\newcommand{\bassumption}{\begin{assumption} }
\newcommand{\eassumption}{\hfill \rule{1mm}{2mm}
\end{assumption} }

\newcommand{\balgorithm}{\medskip\begin{algorithm} \rm}
\newcommand{\ealgorithm}{ \hfill \rule{1mm}{2mm}\medskip
\end{algorithm} }

\newcommand{\basm}{\begin{assumption} \begin{rm} }
\newcommand{\easm}{ \end{rm} \hfill\rule{1mm}{2mm}
\end{assumption} }

\begin{document}

\title{MOFM-Nav: On-Manifold Ordering-Flexible Multi-Robot Navigation}

\author{
\IEEEauthorblockN{Bin-Bin Hu, Weijia Yao, Ming Cao
}
\thanks{
Bin-Bin Hu and Ming Cao are with the Engineering and Technology Institute Groningen, Faculty of Science and Engineering, University of Groningen, 9747 AG Groningen, The Netherlands (e-mail: b.hu@rug.nl, m.cao@rug.nl).
}
\thanks{Weijia Yao is with the School of Robotics, Hunan University, Hunan 410082, P.R.~China. (e-mail:
wjyao@hnu.edu.cn). }
}

\maketitle

\begin{abstract}
This paper addresses the problem of multi-robot navigation where robots maneuver on a desired \(m\)-dimensional (i.e., \(m\)-D) manifold in the $n$-dimensional Euclidean space, and maintain a {\it flexible spatial ordering}.
We consider $ m\geq 2$, and the multi-robot coordination is achieved via non-Euclidean metrics. 
Existing work shows that by incorporating $m-1$ arbitrary auxiliary vectors into the traditional guiding vector field (GVF), the $m$-D manifold can be embedded in an $(n+m)$-dimensional Euclidean space such that the possible singularity of the GVF is eliminated. 
However, since the $m$-D manifold can be characterized by the zero-level sets of $n$ implicit functions, the last $m$ entries of the GVF propagation term become {\it strongly coupled} with the partial derivatives of these functions if the auxiliary vectors are not appropriately chosen.
These couplings not only influence the on-manifold maneuvering of robots, but also pose significant challenges to the further design of the ordering-flexible coordination via non-Euclidean metrics. 

To tackle this issue, we first identify a feasible solution of auxiliary vectors 
such that the last $m$ entries of the propagation term are effectively decoupled to be the same constant. Then, we redesign the coordinated GVF (CGVF) algorithm to {\it boost} the advantages of singularities elimination and global convergence by treating $m$ manifold parameters as additional $m$ virtual coordinates. Furthermore, we enable the on-manifold ordering-flexible motion coordination by allowing each robot to share $m$ virtual coordinates with its time-varying neighbors and a virtual target robot, which {\it circumvents} the possible complex calculation if Euclidean metrics were used instead. Finally, we showcase the proposed algorithm's flexibility, adaptability, and robustness through extensive simulations with different initial positions, higher-dimensional manifolds, and robot breakdown, respectively.
\end{abstract}


\IEEEpeerreviewmaketitle
\section{Introduction}

Robot navigation on manifolds is essential in the robotic research, offering broad applications in environmental monitoring, target escorting, pipeline inspection, and celestial exploration, owing to the manifolds' features of capturing precise geometric characteristics, handling nonlinear constraints, and increasing navigation efficiency \cite{rayhana2021valve,dunbabin2012robots,hu2024ordering}. Among these scenarios, on-manifold multi-robot navigation has emerged as a prime research topic, leveraging the efficiency, extended range, and resilience of multi-robot systems~\cite{fan2020distributed}.

\subsection{Related Works}
 
The pioneering work of on-manifold multi-robot navigation can be traced back to multi-robot path navigation (i.e., $1$-D manifold) by the projection-point~\cite{aguiar2007trajectory}, 
line-of-sight (LOS) \cite{rysdyk2006unmanned} and guiding-vector-field (GVF) methods \cite{kapitanyuk2017guiding}, etc. Since then, more efforts have been emphasized on designing strategies for tackling different kinds of paths, such as adaptive control for straight paths \cite{borhaug2010straight}, virtual structures for sinusoidal paths \cite{ghommam2010formation}, nested invariant set and persistent exciting conditions for circular paths \cite{doosthoseini2015coordinated,hu2021bearing}, a dual-step design for general parameterized paths \cite{ghabcheloo2009coordinated} and hybrid control for nonparameterized paths \cite{lan2011synthesis}.
For more complex paths in $3$-D Euclidean space, significant endeavors have also been devoted to periodic-changing closed paths by output regulation theory \cite{sabattini2015implementation}, simple closed paths by artificial vector fields \cite{pimenta2013Decentralized}, and even self-intersecting paths by higher-dimensional GVF methods \cite{hu2023spontaneous} with an additional virtual coordinate. Despite the significant advancements in the aforementioned works \cite{borhaug2010straight,ghommam2010formation,doosthoseini2015coordinated,hu2021bearing,ghabcheloo2009coordinated,lan2011synthesis,sabattini2015implementation,pimenta2013Decentralized,hu2023spontaneous}, they were largely limited to simple 1-D manifold (i.e., path) navigation missions. Multi-robot navigation on higher-dimensional manifolds (i.e., $m$-D manifold, $m\geq 2$) remains insufficiently explored, even if we only restrict our attention to 2-D manifolds. A major challenging problem is how to coordinate the motions of multiple robots because the coordination among robots may only rely on non-Euclidean metrics rather than Euclidean metrics on $m$-D manifold ($m\geq 2$).

In parallel, there also exist several works exploring higher-dimensional on-manifold multi-robot navigation. For instance, a tangents-to-geodesics-based gradient approach was proposed in \cite{bhattacharya2014multi} to achieve multi-robot coverage and exploration on $2$-D Riemannian manifolds. A graph-theoretical interpretation was developed in \cite{montenbruck2017fekete} to govern robots to move to a given $2$-D sphere with the desired formation asymptotically. However, these two works \cite{bhattacharya2014multi,montenbruck2017fekete} relied on coordinates from the manifold or the Euclidean space, both of which are challenging to compute. Later, a coordinated GVF algorithm was proposed in \cite{yao2022guiding,hu2024coordinated}  for $2$-D surface navigation missions, where the coordination among robots on the manifold was elegantly achieved through the interaction of two additional virtual coordinates. Nevertheless, the aforementioned studies \cite{bhattacharya2014multi,montenbruck2017fekete,yao2022guiding,hu2024coordinated} only focused on $2$-D manifolds, where the multi-robot navigation on general higher-dimensional manifolds remains an open question.

Another crucial question is how to achieve ordering-flexible coordination on manifolds, which can enhance flexibility, adaptability, and navigation efficiency in dynamic environments. Most of the aforementioned navigation works \cite{borhaug2010straight,ghommam2010formation,doosthoseini2015coordinated,hu2021bearing,ghabcheloo2009coordinated,sabattini2015implementation,pimenta2013Decentralized,bhattacharya2014multi,montenbruck2017fekete,yao2022guiding} rely on fixed-ordering coordination, where the desired formation pattern is predetermined for each robot with fixed spatial orderings. This approach is inefficient and may even fail in changing environments. While the introduction of virtual coordinates in the previous GVF design \cite{hu2023spontaneous,hu2024coordinated} provides a promising idea for coordinating robots on $1$-D and $2$-D manifolds conveniently, it cannot be directly extended to general higher-dimensional manifolds. 
The main challenge lies in the suitable decoupling of the propagation term when integrating $m-1$ arbitrary auxiliary vectors into the traditional GVF design \cite{hu2023spontaneous,hu2024coordinated}. In particular, since the $m$-dimensional manifold can be characterized by the zero-level sets of implicit functions, the last $m$ entries of the propagation term of the GVF become strongly coupled with the partial derivative of the implicit functions if the auxiliary vectors are poorly chosen.
This complexity presents a significant challenge to the further design of ordering-flexible navigation on the general manifold $m\geq 2$.

\subsection{Contribution}

To bridge the gap between ordering-flexible coordination on the $m$-D manifold (i.e., $m\geq 2$) and suitable propagation decoupling in the GVF design, we first identify a feasible solution of $m-1$ auxiliary vectors such that the last $m$ entries of the propagation term are effectively decoupled to be the same constant. Then, we redesign the CGVF algorithm for $m$-D manifold to boost the advantages of singularities elimination and global convergence by treating $m$ manifold parameters as additional $m$ virtual coordinates. Furthermore, we enable the on-manifold ordering-flexible coordination by allowing each robot to share the virtual coordinates with its time-varying neighbors and a virtual target robot. Finally, we showcase the flexibility, adaptability, and robustness of the proposed algorithm through extensive simulation. The main contributions are summarized in three-fold.
\begin{enumerate}
  
\item We identify a feasible solution consisting of $m-1$ auxiliary vectors to decouple the last $m$ entries of the propagation term to be the same constant in GVF, paving the way for the ordering-flexible coordination design on the $m$-D manifold ($m \geq 2$).

\item We design a distributed CGVF algorithm for a team of robots to achieve the ordering-flexible navigation maneuvering on the desired $m$-D manifold (i.e., {\it MOFM-Nav}). 

\item We ensure global convergence by treating $m$ manifold parameters as additional $m$ virtual coordinates in the CGVF redesign, and conveniently enable on-manifold coordination with flexible orderings by sharing only $m$ virtual coordinates (i.e., non-Euclidean metrics) among time-varying neighbors and a virtual target robot.


\end{enumerate}

\subsection{Notation}
We denote the real, positive, and integer numbers to be $\mathbb{R}, \mathbb{R}^+, \mathbb{Z}$, respectively. We denote the integer set $\{m\in\mathbb{Z}~|~ i\leq m\leq j\}$ by $\mathbb{Z}_i^j, \forall i\leq j \in\mathbb{Z}$. If a function $f$ is twice continuously differentiable, we denote by $f\in\mathcal{C}^2$. Given a signal $\bold{z}$ in the $n$-dimensional Euclidean space, we denote by $\bold{z}\in\mathbb{R}^n$. We denote the norm of the signal $\bold{z}$ by $\|\bold{z}\|$. We denote an $n$-dimensional identity matrix by $I_n$. The $m\times n$-dimensional matrix whose terms are all zeros is denoted by $\textbf{0}_{m\times n}$.

The remainder of the paper is organized as follows. In Section~\ref{sec_preliminaries}, we describe the necessary preliminaries for on-manifold ordering-flexible navigation. In Section~\ref{sec_problem_formulation}, we formulate the {\it MOFM-Nav} problem. In Section~\ref{secion_3_results}, we first introduce the feasible auxiliary vectors for decoupling the propagation term (see Lemma~\ref{lemma_existance_nu}) and then design the CGVF controller. In Section~\ref{sec_convergence}, we show the convergence results. The first result (see Lemma~\ref{lemma_well_define}) prevents all the ill-defined points of the proposed CGVF algorithm. The second result (see Lemma~\ref{lemma_manifold_convergence}) shows that all the robots converge to the desired manifold. The third result (see Lemmas~\ref{lemma_manifold_maneuvering}-\ref{lemma_manifold_OF_coordination}) indicates the maneuvering and ordering-flexible motion coordination of robots on the manifold. In Section~\ref{section_algorithm}, we conduct extensive simulations for algorithm verification. Finally, we draw the conclusion in Section~\ref{section_conclusion}.

\section{Preliminaries}
\label{sec_preliminaries}
In this section, we will introduce the necessary building blocks for the {\it MOFM-Nav} missions

\subsection{Multi-Robot System}
The first building block is a team of robots with the index $\mathcal V=\{1, \cdots, N\}$, where each robot is governed by the first-order dynamics
\begin{align}
\label{robot_dynamic}
\dot{\bold{x}}_i=\bold{u}_i, i\in\mathcal V,
\end{align}
where $\bold{x}_i=[x_{i,1}, \cdots, x_{i,n}]\t\in\mathbb{R}^n, \bold{u}_i=[u_{i,1}, \cdots, u_{i,n}]\t\in\mathbb{R}^n,$ are the position and input of robot $i, i\in\mathcal V$, respectively. The reason for considering the simple first-order dynamics in~\eqref{robot_dynamic} is that we focus on the upper-level, on-manifold multi-robot navigation rather than the low-level vehicle locomotion. When dealing with complex robot dynamics, such as ground vehicles, fixed-wing aircraft, and unmanned surface vessels \cite{wei2023hierarchical,yao2022guiding}, the velocity input $\bold{u}_i$ can be conveniently treated as a high-level command, where the overall performance can still be guaranteed via a hierarchical control~structure \cite{hu2024ordering,hu2024coordinated}.

\subsection{$m$-D Manifold}
The second building block is an $m$-D desired manifold $\mathcal M$ in the $n$-dimensional Euclidean space, which can be characterized by a set of $n$ zero-level parametric functions $\phi_i$ \cite{hu2024coordinated},
\begin{align}
\label{desired_manifold}
\mathcal M^{phy}=&\{\bm\sigma:=[\sigma_1, \cdots, \sigma_n]\in\mathbb{R}^n~|~\phi_j(\bm\sigma):=\sigma_j- \nonumber\\
&f_j(\omega_{1}\cdots, \omega_{m})=0, j\in\mathbb{Z}_1^n, \omega_1,\dots, \omega_m \in \mathbb{R}\},
\end{align}
where $\sigma_1, \cdots, \sigma_n\in\mathbb{R}$ are $n$ coordinates, $\omega_{1}, \cdots, \omega_{m}\in\mathbb{R}$ are $m$ parameters of the manifold, and $\phi_j(\cdot): \mathbb{R}\rightarrow \mathbb{R}$ are twice continuously differentiable, i.e., $\phi_j\in \mathcal C^2$. Note that $m$ virtual coordinates in \eqref{desired_manifold} determine the $m$-dimensional manifold $\mathcal M$, and are independent of the dimensions of the Euclidean space where they live. For instance, $\mathcal M^{phy}$ is a desired $0$-D point if $m=0$, a $1$-D path $\mathcal P^{phy}$ if $m=1$, and a $2$-D surface $\mathcal S^{phy}$ if $m=2$ \cite{lee2010introduction}. Additionally, if $m\geq 2$, the traditional Euclidean metric becomes unsuitable for achieving the on-manifold navigation and coordination on the $m$-D manifold anymore, while these two tasks only rely on the complex geometric (non-Euclidean) metrics. It is still worth mentioning that not all the parametric functions in \eqref{desired_manifold} can be utilized to illustrate the $m$-D desired manifold, where the additional assumption is given in~{\bf A1} later.

\subsection{Higher-Dimentional GVF for $m$-D Manifold}
To avoid using the complex geometric metrics directly on the $m$-D manifold, we introduce the third building block of the higher-dimensional GVF (HGVF) via implicit functions and $m$ parameters (i.e., $\phi_j$ and $\omega_{1}, \cdots, \omega_{m}$ in \eqref{desired_manifold}). 

Analogous to \eqref{desired_manifold}, suppose that the $i$-th desired (physical) $m$-D manifold $\mathcal M_i^{phy}$ for robot $i, i\in\mathcal V$, to converge in the $n$-denmensional Euclidean space is characterized by
\begin{align}
\label{ith_desired_manifold}
\mathcal M_i^{phy}=&\{\bm\sigma_i:=[\sigma_{i,1}, \cdots, \sigma_{i,n}]\t\in\mathbb{R}^n~|~\phi_{i,j}(\bm\sigma_{i}):=\nonumber\\
&\sigma_{i,j}-f_{i,j}(\omega_{i,1}, \cdots, \omega_{i,m})=0, j\in\mathbb{Z}_1^n,\omega_1,\nonumber\\
&\dots, \omega_m \in \mathbb{R}\},
\end{align}
where $\bm\sigma_i\in\mathbb{R}^n$ is the coordinates of robot $i$, $\phi_{i,j}\in\mathbb{R}, i\in\mathcal V, j\in\mathbb{Z}_1^n$ are the implicit functions, $f_{i,j}(\omega_{i,1}, \cdots, \omega_{i,m})\in\mathcal C^2, i\in\mathcal V, j\in\mathbb{Z}_1^n$ are twice continuously differentiable implicit functions, $\omega_{i,1}\in\mathbb{R}, \cdots, \omega_{i,m}\in\mathbb{R}$ are the virtual coordinates. Let $\bm\xi_i:=[\sigma_{i,1}, \cdots, \sigma_{i,n}, \omega_{i,1}, \cdots, \omega_{i,m}]\t\in\mathbb{R}^{n+m}$ be the generalized coordinate vector of robot $i$, one has that the higher-dimensional manifold $\mathcal M_i^{hgh}$ embeded in the $n+m$-dimensional Euclidean space becomes
\begin{align}
\label{high_ith_desired_manifold}
\mathcal M_i^{hgh}=&\{\bm\xi_i:=[\sigma_{i,1}, \cdots, \sigma_{i,n}, \omega_{i,1}, \cdots, \omega_{i,m}]\t\in\mathbb{R}^{n+m}~|~\nonumber\\
&\phi_{i,j}(\bm\xi_i):=\sigma_{i,j}-f_{i,j}(\omega_{i,1}, \cdots, \omega_{i,m})=0, \nonumber\\
&j\in\mathbb{Z}_1^n, \omega_1, \dots, \omega_m \in \mathbb{R}\}.
\end{align}

\begin{definition}
\label{def_cross_product}
(Generalized cross product \cite{galbis2012vector}) Let $\mathbf{g}_{1}\in \mathbb{R}^n$, $ \cdots,  \mathbf{g}_{n-1}\in \mathbb{R}^n$ be $n-1$ linearly independent vectors with $\mathbf{g}_i:=[g_{i,1}, \cdots, g_{i,n}]\t, i\in\mathbb{Z}_1^n$, the generalized cross product $\times(\mathbf{g}_{1}, \cdots, \mathbf{g}_{n-1})$ is calculated by
\begin{align*}
\times(\mathbf{g}_{1}, \cdots, \mathbf{g}_{n-1})
=&\sum_{j=1}^n(-1)^{j-1}\mathrm{det}(G_i)\mathbf{b}_{j},
\end{align*}
where $\times: \mathbb{R}^{n}\times \cdots \times \mathbb{R}^{n}\rightarrow\mathbb{R}^{n}$ denote the cross product, $\bold{b}_j=[0, \cdots, 1, \cdots, 0]\t\in\mathbb{R}^{n}$ is the basis column vector with the $j$-th component being $1$ and the other components being $0$, and $\mathrm{det}(G_j), j\in\mathbb{Z}_{1}^{n}$ represent the determinants of the submatrices $G_j\in\mathbb{R}^{(n-1)\times (n-1)},$ by deleting $j$-th column of the matrix of $[\mathbf{g}_{1}, \cdots, \mathbf{g}_{n-1}]\t\in\mathbb{R}^{(n-1)\times n}$ below
\begin{align*}
&\mathrm{det}(G_j)=\nonumber\\
&
\left|\begin{array}{cccccc} 
g_{1,1} &    \cdots & g_{1, j-1} &  g_{1,j+1} & \cdots & g_{1,n}  \\ 
g_{2,1} &    \cdots & g_{2, j-1} &  g_{2,j+1} & \cdots & g_{2,n}  \\ 
    \vdots &    \ddots  &  \vdots &   \vdots &  \ddots&  \vdots \\ 
g_{n-2,1} &    \cdots & g_{n-2, j-1} &  g_{n-2,j+1} & \cdots & g_{n-2,n}  \\ 
g_{n-1,1} &    \cdots & g_{n-1, j-1} &  g_{n-1,j+1} & \cdots & g_{n-1,n}  \\ 
\end{array}\right|. 
\end{align*} 
\end{definition}

Using the generalized cross product in Definition~\ref{def_cross_product}, we can introduce the definition of HGVF for the $m$-D manifold below.
\begin{definition}
\label{def_GVF_manifold}
(HGVF for $m$-D manifold) Given the $i$-th manifold $\mathcal M_i^{hgh}$ in \eqref{high_ith_desired_manifold} for robot $i$, the corresponding HGVF $\bm\chi_i^{hgh}\in\mathbb{R}^{n+m}$ is designed to be,
\begin{align}
\label{ith_GVF}
\bm\chi_i^{hgh}=&\times(\nabla\phi_{i,1}(\bm\xi_i), \cdots, \nabla\phi_{i,n}(\bm\xi_i), \bm\nu_{i}^{[1]}, \cdots, \bm\nu_{i}^{[m-1]})\nonumber\\
	 	    &-\sum_{j=1}^n k_{i,j}\phi_{i,j}(\bm\xi_i) \nabla\phi_{i,j}(\bm\xi_i),
\end{align}
where $\times: \overbrace{\mathbb{R}^{n+m}\times \cdots \times \mathbb{R}^{n+m}}^{n+m-1}\rightarrow\mathbb{R}^{n+m}$ represents the generalized cross product in Definition~\ref{def_cross_product}, and $k_{i,j}\in\mathbb{R}^{+}, j\in\mathbb{Z}_1^n$ are the control gains for the convergence term. The vector $\nabla\phi_{i,j}(\bm\xi_i): \mathbb{R}^{n+m}\rightarrow\mathbb{R}^{n+m}$ denotes the gradient of $\phi_{i,j}(\bm\xi_i)$ w.r.t $\bm\xi_i$ in \eqref{high_ith_desired_manifold}, i.e.,
\begin{align}
\label{gradient_phi_ij}
\nabla\phi_{i,j}(\bm\xi_i):=[\overbrace{0, \cdots, 1, \cdots, 0}^n, \overbrace{ -\partial f_{i,j}^{[1]}, \cdots, -\partial f_{i,j}^{[m]}}^m]\t
\end{align} 
with the $j$-th component of the gradient vector $\nabla\phi_{i,j}(\bm\xi_i)$ being~$1$, and
\begin{align}
\label{eq_partial_fij}
\partial f_{i,j}^{[l]}=\frac{\partial f_{i,j}(\omega_{i,1}, \cdots, \omega_{i,m})}{\partial \omega_{i,l}}, l\in\mathbb{Z}_1^m
\end{align} 
being the partial derivative of $f_{i,j}(\omega_{i,1}, \cdots, \omega_{i,m})$ in \eqref{ith_desired_manifold} w.r.t $l$-th virtual coordinates $\omega_{i,l}$ of robot $i$. Note that 
\begin{align}
\label{auxiliary_vector}
\bm\nu_{i}^{[1]} = &[\nu_{i,1}^{[1]}, \dots, \nu_{i,n+m}^{[1]}] \in \mathbb{R}^{n+m}, \\
\vdots~~=&~~~~~~~~~~~~~~\vdots\nonumber\\
\bm\nu_{i}^{[m-1]} =& [\nu_{i,1}^{[m-1]}, \dots, \nu_{i,n+m}^{[m-1]}] \in \mathbb{R}^{n+m},
\end{align} are $m-1$ auxiliary constant vectors of robot~$i$, which are employed to ensure that the propagation term $\times(\nabla\phi_{i,1}(\bm\xi_i), \dots$, $\nabla\phi_{i,n}(\bm\xi_i), \bm\nu_{i}^{[1]}, $ $\dots, \bm\nu_{i}^{[m-1]})$ in \eqref{ith_GVF} is well-defined.
\end{definition}

In Definition~\ref{def_GVF_manifold}, the HGVF $\bm\chi_i^{hgh}$ in~\eqref{ith_GVF} can be conveniently regarded as an upper-level desired velocity for robot $i$ of higher-order dynamics. Therein, the first propagation term  $\times(\nabla\phi_{i,1}(\bm\xi_i), \cdots$, $ \nabla\phi_{i,n}(\bm\xi_i), \bm\nu_{i}^{[1]}, \cdots, \bm\nu_{i}^{[m-1]})$ in \eqref{ith_GVF}, orthogonal to all the gradients and auxiliary vectors $\nabla\phi_{i,j}(\bm\xi_i), j\in\mathbb{Z}_1^n, \bm\nu_{i}^{[1]}, \cdots, \bm\nu_{i}^{[m-1]}$ in \eqref{gradient_phi_ij} and \eqref{auxiliary_vector}, respectively, can keep robot $i$ maunver on its desired manifold $\mathcal M_i^{hgh}$, and the second convergence term $-\sum_{j=1}^n k_{i,j}\phi_{i,j}(\bm\xi_i) \nabla\phi_{i,j}(\bm\xi_i)$ in \eqref{ith_GVF} can guide robot~$i$ to converge to the desired manifold $\mathcal M_i^{hgh}$. Essentially, by introducing $\omega_{i,1}, \cdots, \omega_{i,m}$ as additional $m$ virtual coordinates and projecting the HGVF $\bm\chi_i^{hgh}$ in \eqref{ith_GVF} back to the first $n$-dimensional Euclidean space, the designed $\bm\chi_i^{hgh}$ can eliminate the singular points (i.e., $\bm\chi_i^{hgh}\neq\bold{0}_{m+n}$) and ensure the global convergence to the original $\mathcal M_i^{phy}$ in~\eqref{ith_desired_manifold} \cite{yao2021singularity}.


However, the designed HGVF $\bm\chi_i^{hgh}$ in \eqref{ith_GVF} only features the on-manifold navigation (i.e., convergence and maneuvering) for an individual robot, which has not touched the scenario of on-manifold robot-robot coordination. Later, we will show that this coordination can be conveniently achieved through the interaction of non-Euclidean metrics (i.e., $m$ virtual coordinates $\omega_{i,1}, \cdots, \omega_{i,m}$), as detailed later in Section~\ref{subsec_manifold_coordination}. A key technical challenge lies in how to find a suitable selection of $m-1$ auxiliary constant vectors $\bm\nu_{i}^{[1]}, \cdots, \bm\nu_{i}^{[m-1]}$, to decouple the propagation term $\times(\nabla\phi_{i,1}(\bm\xi_i), \cdots, \nabla\phi_{i,n}(\bm\xi_i), \bm\nu_{i}^{[1]}, \cdots, \bm\nu_{i}^{[m-1]})$ in \eqref{ith_GVF}, which enables the on-manifold coordination with a flexible ordering. 

%
%
%
%


%
%

\section{Problem Formulation}
\label{sec_problem_formulation}
With the necessary blocks in Section~\ref{sec_preliminaries}, we will first introduce different subtasks, namely, on-manifold convergence \& maneuvering, and ordering-flexible coordination for {\it MOFM-Nav}, and then formulate the problem.

\subsection{On-Manifold Multi-Robot Convergence \& Maneuvering}
Let $\bold{p}_i:=[\bold{x}_i, \bm\omega_i ]\t\in\mathbb{R}^{n+m}$ with $\bold{x}_i$ given in \eqref{robot_dynamic} and the vector $\bm\omega_i=[\omega_{i,1}, \cdots, \omega_{i,m}]\t\in\mathbb{R}^m$ in \eqref{ith_desired_manifold}. By substituting the position $\bold{x}_i=[x_{i,1}, \cdots, x_{i,n}]\t$ of the $i$-th robot in \eqref{robot_dynamic} for $\bm\sigma_i$ into the $i$-th $\mathcal M_i^{hgh}$ in \eqref{high_ith_desired_manifold}, we obtain the multi-robot on-manifold convergence errors 
\begin{align}
\label{subtask_manifold_convergence}
\phi_{i,j}(\bold{p}_i)=x_{i,j}-f_{i,j}(\bm\omega_i), j\in\mathbb{Z}_1^n.
\end{align}
Hence, the on-manifold multi-robot convergence is achieved if the following condition is satisfied, 
\begin{align}
\label{condition_convergence}
\lim_{t\rightarrow\infty}\phi_{i,j}(\bold{p}_i(t))=0, \forall i\in\mathcal V, j\in\mathbb{Z}_1^n. 
\end{align}

When each robot $i$ governed by $\bm\chi_i^{hgh}$ in \eqref{ith_GVF} approaches the manifold $\mathcal M_i^{hgh}$ (i.e., \eqref{condition_convergence} is fulfilled), the convergence term in \eqref{ith_GVF} becomes zero, i.e., $-\sum_{j=1}^n k_{i,j}\phi_{i,j}(\bold{p}_i) \nabla\phi_{i,j}(\bold{p}_i)=\bold{0}_{n+m}$, which implies that $\bm\chi_i^{hgh}$ in \eqref{ith_GVF} becomes 
\begin{align}
\label{degenarated_model}
\bm\chi_i^{hgh}=\times(\nabla\phi_{i,1}(\bold{p}_i), \cdots, \nabla\phi_{i,n}(\bold{p}_i), \bm\nu_{i}^{[1]}, \cdots, \bm\nu_{i}^{[m-1]}).
\end{align} 
Combining with the definition of $\bm\omega_i:=[\omega_{i,1}, \cdots, \omega_{i,m}]$ which is given in \eqref{subtask_manifold_convergence}, one has that the new HGVF $\bm\chi_i^{hgh}$ in \eqref{degenarated_model} is only determined by $\bm\omega_i$ and $\bm\nu_{i}^{[1]}, \cdots, \bm\nu_{i}^{[m-1]}$, provided that robot $i$ is already on the manifold. 

Furthermore, since $\times(\nabla\phi_{i,1}(\bold{p}_i), \cdots$, $\nabla\phi_{i,n}(\bold{p}_i), \bm\nu_{i}^{[1]}, \cdots$, $ \bm\nu_{i}^{[m-1]})$ in \eqref{degenarated_model} is responsible for the on-manifold robot maneuvering in Definition~\ref{def_GVF_manifold}, and $\dot{\bm\omega}_i$ represent the desired velocities of the additional $m$ virtual coordinates (i.e., last $m$ elements of $\bm\chi_i^{hgh}$ in~\eqref{ith_GVF} with $\dot{\bold{p}}_i:=[\dot{\bold{x}}_i, \dot{\bm\omega}_i ]\t=\bm\chi_i^{hgh}$), one has that the second subtask of on-manifold maneuvering for each robot is achieved if the following condition is satisfied,
\begin{align}
\label{condition_maneuvering}
\lim_{t\rightarrow\infty}\dot{\bm\omega}_{i}(t)=\lim_{t\rightarrow\infty}\dot{\bm\omega}_{k}(t)\neq\bold{0}_m, \forall i\neq k\in\mathcal V,
\end{align}
where the constant auxiliary vectors $\bm\nu_{i}^{[1]}, \cdots, \bm\nu_{i}^{[m-1]}$ need to be identical for All the robots later.

\subsection{On-Manifold Ordering-Flexible Coordination}
\label{subsec_manifold_coordination}

Recalling $\lim_{t\rightarrow\infty}\dot{\bm\omega}_{i}(t)=\lim_{t\rightarrow\infty}\dot{\bm\omega}_{k}(t)$, $\forall i\neq k\in\mathcal V$ in \eqref{condition_maneuvering}, one has that the relative virtual coordinates $\bm\omega_i, \bm\omega_k$ against each robot will become invariant, i.e.,
\begin{align*}
\lim_{t\rightarrow\infty}\bm\omega_{i}(t)-\lim_{t\rightarrow\infty}\bm\omega_{k}(t)=\bold{d}_{i,k}, \forall i\neq k\in\mathcal V,
\end{align*}
for a constant vector $\bold{d}_{i,k}\in\mathbb{R}^{m}$. By mapping $\bm\omega_i(t)$  back onto the common manifold $\mathcal{M}_i^{phy}$, i.e., $f_{i,j}(\omega_{i,j}) \rightarrow x_{i,j}, \forall i \in \mathcal{V}, j \in \mathbb{Z}_1^n$, one has that the displacement among robots $i,k$ on the manifold will maintain as well, albeit with some deformation. However, it still lacks the last element of ordering-flexible coordination.

To proceed, we introduce a virtual target robot labelled by $\ast$, governed by $\bm{\chi}_{\ast}^{hgh}$ in \eqref{ith_GVF}, already maneuvering on the common manifold $\mathcal{M}_{\ast}^{hgh}$ in \eqref{high_ith_desired_manifold} with the target virtual coordinates $\bm{\omega}_{\ast}:=[\omega_{\ast,1}, \cdots, \omega_{\ast,m}]$ and the positions $\bold{p}_{\ast}(t)$, respectively. Note that the on-manifold convergence errors of the virtual robot $\ast$ are zeros, i.e., 
$$\phi_{*,j}(\bold{p}_{\ast}(t))=0, \forall j\in\mathbb{Z}_1^n.$$
Analogous to \eqref{condition_maneuvering}, the derivative $\dot{\bm\omega}_{\ast}$ satisfies 
\begin{align}
\label{condition_target_virtual_coordinate}
\lim_{t\rightarrow\infty}\{\dot{\bm\omega}_{i}(t)-\dot{\bm\omega}_{\ast}(t)\}=\bold{0}_m, \forall i\in\mathcal V,
\end{align}
under the same vectors $\bm\nu_{\ast}^{[1]}, \cdots, \bm\nu_{\ast}^{[m-1]}$ as other robots. This implies that the relative virtual coordinates $\bm\omega_i-\bm\omega_\ast$ between robot $i$ and the virtual target robot $\ast$ are also invariant.
Then, we define the sensing neighborhood of robot $i$ via virtual coordinates below
\begin{align}
\label{sensing_neighbor}
\mathcal N_i=\{k\in\mathcal V, k\neq i~\big |~ \|\bm\omega_i-\bm\omega_k\|\leq R\},
\end{align}
where $R\in \mathbb{R}^+$ is a predefined sensing radius on the $m$ virtual coordinates $\bm\omega_i$.

According to $\bm\omega_{i}, \bm\omega_{\ast}, \mathcal N_i$ in \eqref{condition_target_virtual_coordinate} and \eqref{sensing_neighbor}, one has that the on-manifold ordering-flexible coordination can be determined by the following two conditions of the virtual coordinates,
\begin{align}
\label{condition_ordering_flexible}
&(a)~\lim_{t\rightarrow\infty}\frac{1}{N}\sum_{i=1}^N\bm\omega_i(t)-\bm\omega_{\ast}(t)=\bold{0}_m,\nonumber\\
&(b)~r<\lim_{t\rightarrow\infty}\|\bm\omega_{i}(t)-\bm\omega_{k}(t)\|<R, i\in\mathcal V, \forall k\in\mathcal N_i,
\end{align}
where $R$ is given in \eqref{sensing_neighbor} and $r\in(0, R)$ is the safe radius. The condition (a) represents the attraction of virtual coordinates to the virtual target robot, while the condition (b) indicates the repulsion of virtual coordinates among neighboring
robot. Additionally, the left side of the inequality (a): $\lim_{t\rightarrow\infty}\|\bm\omega_{i}(t)-\bm\omega_{k}(t)\|>r$ also prevents the robot overlapping (i.e., $\exists i\neq k, \bm\omega_{i}=\bm\omega_{k}$) on the manifold.

Since the ordering sequence of the coordination is not predefined and fixed in \eqref{condition_ordering_flexible} and $\mathcal N_i$ is time-varying, one has that the on-manifold ordering-flexible coordination can be achieved. Particularly, this flexible coordination in \eqref{condition_ordering_flexible} circumvents using the complicated calculated geodesic distance on the manifold.  Alternatively, we leverage the simple interaction between different virtual coordinates $\bm\omega_i$ and then map back to the manifold $\mathcal M_i^{phy}$, which can reduce communication costs but may cause challenging nonlinear coupling in convergence analysis. 



\begin{definition}
\label{def_MOFM_navigation}
({\it MOFM-Nav}) A team of robots $\mathcal V$ governed by \eqref{robot_dynamic} collectively achieve the ordering-flexible coordination while maneuvering on the desired common manifold $\mathcal M_i^{phy}, i\in\mathcal V$ governed by \eqref{ith_desired_manifold} if the following conditions are fulfilled,

\begin{itemize}
\item {\bf C1 (On-manifold convergence)}: All the robots $\mathcal V$ converge to the desired common manifold $\mathcal M_i^{phy}$, i.e., $\lim_{t\rightarrow\infty}\phi_{i,j}(\bold{p}_i(t))=0, \forall i\in\mathcal V, j\in\mathbb{Z}_1^n$ in \eqref{condition_convergence}. 

\begin{figure}[!htb]
\centering
\includegraphics[width=7cm]{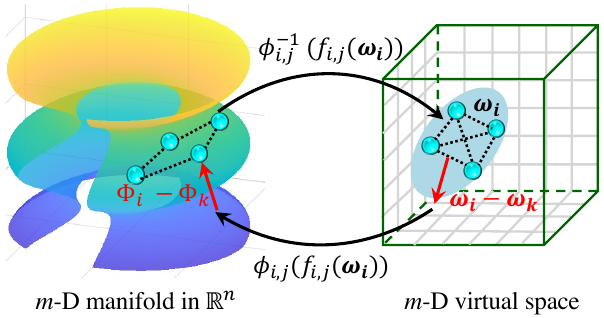}
\caption{Illustration of the core concept of {\it MOFM-Nav} via non-Euclidean metrics.  Note that the coordination of robots is first mapped via $\phi_{i,j}^{-1}(f_{i,j}(\bm\omega_i))$ and implicitly achieved in the $m$-virtual coordinate space of $\bm\omega_i$ (right part), and then mapped back via $\phi_{i,j}(f_{i,j}(\bm\omega_i))$ in the real space where the $m$-D manifold lives (left part).} 
\label{Manifold_illustration}
\end{figure}
\item {\bf C2 (On-manifold maneuvering)}: All the robots $\mathcal V$  maneuver along the desired common manifold $\mathcal M_i^{phy}$, i.e., $\lim_{t\rightarrow\infty}\dot{\bm\omega}_{i}(t)=\lim_{t\rightarrow\infty}\dot{\bm\omega}_{k}(t)\neq\bold{0}_m, \forall i\neq k\in\mathcal V$ in \eqref{condition_maneuvering}. 

\item {\bf C3 (Ordering-flexible coordination)}: All the robots coordinate the ordering-flexible pattern on the manifold via the interaction of virtual coordinates: (a) $\lim_{t\rightarrow\infty}$ ${1}/{N}\sum_{i=1}^N\bm\omega_i(t)-\bm\omega_{\ast}(t)=\bold{0}_m$, (b) $r<\lim_{t\rightarrow\infty}$ $\|\bm\omega_{i}(t)-\bm\omega_{k}(t)\|$ $<R, i\in\mathcal V, \forall k\in\mathcal N_i$ in \eqref{condition_ordering_flexible}.
\end{itemize}
\end{definition}

In Definition~\ref {def_MOFM_navigation}, {\bf C1-C2} account for the on-manifold navigation of an individual robot, while the core idea of {\bf C3} is the target-robot attraction and robot-robot repulsion reaching a balance on the virtual coordinates. An illustration of the {\it MOFM-Nav} via non-Euclidean metrics is given in Fig.~\ref{Manifold_illustration}.
%

\subsection{Problem Formulation}
Recalling the definition of on-manifold convergence errors $\phi_{i,j}(\bold{p}_i)$ which is given by \eqref{subtask_manifold_convergence}, one has that $\dot{\phi}_{i,j}(\bold{p}_i)=\nabla\phi_{i,j}(\bold{p}_i)\t\dot{\bold{p}}_i$ with $\nabla\phi_{i,j}(\bold{p}_i)$ in \eqref{gradient_phi_ij}. 
Meanwhile, let $\Phi_i(\bold{p}_i)=[\phi_{i,1}(\bold{p}_i)$, $\cdots, \phi_{i,n}(\bold{p}_i)]\t \in\mathbb{R}^{n}$ and $\bold{u}_i^{\omega}:=[u_{i,1}^{\omega}, \cdots, u_{i,m}^{\omega}]\t=\dot{\bm\omega}_i\in\mathbb{R}^m$ be the input of the virtual coordinates $\bm\omega_i$, it follows from \eqref{robot_dynamic} that the closed-loop system for robot~$i, i\in\mathcal V$ becomes

\begin{tikzpicture}[overlay, remember picture]
     \node[anchor=west] at (6.7, -2.9) {\rotatebox{180}{$\left\{ \begin{array}{c} \\ \\ \\  \\ \end{array} \right.$}};
     \node[anchor=west] at (6.7, -4.3) {\rotatebox{180}{$\left\{ \begin{array}{c} \\ \\   \end{array} \right.$}};
     \node[anchor=west] at (7.45, -2.9) {$\scriptstyle n$};
     \node[anchor=west] at (7.45, -4.3) {$\scriptstyle m$};
\end{tikzpicture}
\begin{align}
\label{GVF_dynamic}
\begin{bmatrix}
\dot{\Phi}_i(\bold{p}_i)\\
\hline
\dot{\bm\omega}_i
\end{bmatrix}=
D_i
\begin{bmatrix}
\bold{u}_i\\
\hline
\bold{u}_i^\omega
\end{bmatrix},
\end{align}
where the dynamic matrix $D_i\in\mathbb{R}^{(n+m)\times(n+m)}$ is
\begin{align}
\label{matrix_D}
D_i=
{\scriptsize \begin{bmatrix}
1 & 0 & \cdots & 0 & -\partial f_{i,1}^{[1]} & \cdots & -\partial f_{i,1}^{[m]}\\
0 & 1 & \cdots & 0 & -\partial f_{i,2}^{[1]} & \cdots & -\partial f_{i,2}^{[m]}\\
\vdots & \vdots & \ddots & \vdots &  \vdots    & \ddots & \vdots \\
0 & 0 & \cdots & 1 & -\partial f_{i,n}^{[1]} & \cdots & -\partial f_{i,n}^{[m]}\\
\hline
0 & 0 & \cdots & 0 & 1 & \cdots & 0\\
\vdots & \vdots & \ddots & \vdots &  \vdots    & \ddots & \vdots \\
0 & 0 & \cdots & 0 & 0 & \cdots & 1\\
\end{bmatrix}}
\end{align}
with $\partial f_{i,j}^{[k]}, j\in\mathbb{Z}_1^n, k\in\mathbb{Z}_1^m$ given in \eqref{eq_partial_fij}.
Now, it is ready to introduce the problem addressed in this paper.

\begin{problem}
\label{problem_ref}
Design the control inputs $\{\bold{u}_i, \bold{u}_i^\omega\}$ of each robot $i, i\in\mathcal V$ to be
\begin{align*}
\{\bold{u}_i, \bold{u}_i^\omega\}=g(\phi_{i,j}(\bold{p}_i), \partial f_{i,j}^{[1]}, \cdots,\partial f_{i,j}^{[m]}, \omega_{i,1}, \cdots, \omega_{i,m})
\end{align*}
such that the closed-loop system governed by \eqref{GVF_dynamic} achieves the {\it MOFM-Nav} (i.e., the conditions {\bf C1-C3} in Definition~\ref{def_MOFM_navigation}.
\end{problem}

To address Problem~\ref{problem_ref}, we still need some assumptions. 

\begin{itemize}

\item {\bf A1} The parametric functions in \eqref{desired_manifold} are assumed to illustrate $m$-D manifold.

\item {\bf A2} For any given $\kappa>0$ and a point  $\bold{x}_0(t)$, it holds that
$\inf\{\|\phi(\bold{x}_0(t))\|: \mbox{dist}(\bold{x}_0,\mathcal M_i^{phy})\geq \kappa\}>0, \forall i\in\mathcal V$.

\item {\bf A3} The initial positions of arbitrary two virtual coordinates satisfy $\|\bm\omega_i(0)-\bm\omega_k(0)\|> r, i\neq k\in\mathcal V$.

\item {\bf A4} The first and second partial derivatives of $f_{i,j}(\omega_{i,1}$, $\cdots, \omega_{i,m}), \forall i\in\mathcal V, j\in\mathbb{Z}_1^n$ w.r.t $m$ virtual coordinates $\omega_{i,1}, \cdots, \omega_{i,m}$ are all bounded.

\item {\bf A5} For the desired common manifold $\mathcal M_i^{phy}$, we assume that there exists a sufficiently large area such that all the robots can be displaced.
\end{itemize}

{\bf A1} assumes that the $m$-D manifold satisfying the homeomorphic property \cite{lee2010introduction} can be illustrated by the implicit functions. {\bf A2} is necessary to ensure that $\lim_{t\rightarrow\infty}\|\phi_{i,j}(\bold{p}_0(t))\|=0\Rightarrow  \lim_{t\rightarrow\infty}\mbox{dist}(\bold{p}_0(t),\mathcal M_i^{phy})=0$. {\bf A3} prevents the overlapping of robots on the manifold initially. {\bf A4} is necessary for the convergence analysis of the {\it MOFM-Nav}. {\bf A5} is reasonable and can be satisfied by shrinking the safe radius $r$. 

\section{Algorithm Design}
\label{secion_3_results}

\subsection{Feasible Auxiliary Vectors for Propagation Decoupling}

According to {\bf C3} in Definition~\ref{def_MOFM_navigation}, to achieve the ordering-flexible coordination in {\it MOFM-Nav}, we need to maintain the relative positions between neighboring virtual coordinates $\bm\omega_i-\bm\omega_k$ to be the same. This inevitably requires that the last $m$ entries of the propagation term $\times(\nabla\phi_{i,1}(\bm\xi_i), \cdots$, $\nabla\phi_{i,n}(\bm\xi_i), \bm\nu_{i}^{[1]}, \cdots, \bm\nu_{i}^{[m-1]})$ in~\eqref{ith_GVF} are decoupled to be same or constant.  However, if the auxiliary vectors $\bm\nu_{i}^{[1]}, \cdots, \bm\nu_{i}^{[m-1]}$ are not appropriately chosen, then the last $m$ entries of the propagation term in~\eqref{ith_GVF} become strongly coupled with the partial derivative of the implicit functions, which poses significant challenges to the design of the ordering-flexible coordination. Therefore, such a random auxiliary vectors $\bm\nu_{i}^{[1]}, \cdots, \bm\nu_{i}^{[m-1]}$ is ``\textbf{infeasible}" for the design of {\it MOFM-Nav}.
To enhance readers' understanding, an undesirable coupling phenomenon is explicitly shown in Example~\ref{couple_example}.


%

\begin{example}
\label{couple_example}
(Undesirable coupling of the propagation term) Given a special 3D manifold with three general gradient vectors in \eqref{gradient_phi_ij} 
\begin{align*}
&\nabla\phi_{i,1}(\bm\xi_i)=[1, 0, 0, -\partial f_{i,1}^{[1]}, -\partial f_{i,1}^{[2]}, -\partial f_{i,1}^{[3]}]\t\in\mathbb{R}^6,\nonumber\\
&\nabla\phi_{i,2}(\bm\xi_i)=[0, 1, 0, -\partial f_{i,2}^{[1]}, -\partial f_{i,2}^{[2]}, -\partial f_{i,3}^{[3]}]\t\in\mathbb{R}^6, \nonumber\\
&\nabla\phi_{i,3}(\bm\xi_i)=[0, 0, 1, -\partial f_{i,3}^{[1]}, -\partial f_{i,3}^{[2]}, -\partial f_{i,3}^{[3]}]\t\in\mathbb{R}^6,
\end{align*}
and two random auxiliary vectors $\bm\nu_{i}^{[1]}=[1, 1, 0, 0, 0, 0]\t\in\mathbb{R}^{6}, \bm\nu_{i}^{[2]}=[0, 0, 0, 0, 1, 1]\t\in\mathbb{R}^{6}$, it follows from the generalized cross product in Definition~\ref{def_cross_product} that the propagation term in~\eqref{ith_GVF} is calculated to be
\begin{align*}
&\times(\nabla\phi_{i,1}(\bm\xi_i),\nabla\phi_{i,2}(\bm\xi_i), \nabla\phi_{i,3}(\bm\xi_i), \bm\nu_{i}^{[1]},  \bm\nu_{i}^{[2]})\nonumber\\
=&[p_{i,1}, p_{i,2}, p_{i,3}, p_{i,4}, p_{i,5}, p_{i,6}]\t
\end{align*}
with 
\begin{align*}
p_{i,1}:=&\partial f_{i,1}^{[2]}\partial f_{i,2}^{[1]}-\partial f_{i,1}^{[1]}\partial f_{i,2}^{[2]}+\partial f_{i,1}^{[1]}\partial f_{i,2}^{[3]}-\partial f_{i,1}^{[3]}\partial f_{i,2}^{[1]}, \nonumber\\
p_{i,2}:=&\partial f_{i,1}^{[2]}\partial f_{i,2}^{[1]}-\partial f_{i,1}^{[1]}\partial f_{i,2}^{[2]}+\partial f_{i,1}^{[1]}\partial f_{i,2}^{[3]}-\partial f_{i,1}^{[3]}\partial f_{i,2}^{[1]}, \nonumber\\
p_{i,3}:=&\partial f_{i,1}^{[1]}\partial f_{i,3}^{[2]}-\partial f_{i,1}^{[2]}\partial f_{i,3}^{[1]}-\partial f_{i,1}^{[1]}\partial f_{i,3}^{[3]}+\partial f_{i,1}^{[3]}\partial f_{i,3}^{[1]}\nonumber\\
&+\partial f_{i,2}^{[1]}\partial f_{i,3}^{[2]}-\partial f_{i,2}^{[2]}\partial f_{i,3}^{[1]}-\partial f_{i,2}^{[1]}\partial f_{i,3}^{[3]}+\partial f_{i,2}^{[3]}\partial f_{i,3}^{[1]}, \nonumber\\
p_{i,4}:=&\partial f_{i,1}^{[2]}-\partial f_{i,1}^{[3]}+\partial f_{i,2}^{[2]}-\partial f_{i,2}^{[3]}, \nonumber\\
p_{i,5}:=&\partial f_{i,1}^{[1]}+\partial f_{i,2}^{[1]},~~p_{i,6}:=\partial f_{i,1}^{[1]}+\partial f_{i,2}^{[1]}.
\end{align*}
\end{example}

In Example~\ref{couple_example}, the last three terms $p_{i,4}, p_{i,5}, p_{i,6}$ in the propagation term $\times(\nabla\phi_{i,1}(\bm\xi_i),\nabla\phi_{i,2}(\bm\xi_i), \nabla\phi_{i,3}(\bm\xi_i)$, $\bm\nu_{i}^{[1]}, \bm\nu_{i}^{[2]})$ are strongly coupled with the different partial derivatives of the implicit functions $\partial f_{i,j}^{[l]}, j\in\mathbb{Z}_1^3, l\in\mathbb{Z}_1^3$ in \eqref{eq_partial_fij} in different dimensions. Therefore, it becomes hard to design ordering-flexible coordination via virtual coordinates. Before proceeding to the ordering-flexible coordination, the first step is to find a ``\textbf{feasible candidate}" for the vectors $\bm\nu_{i}^{[1]}, \cdots, \bm\nu_{i}^{[m-1]}$ such that the last $m$ entries of the propagation term are effectively decoupled to be same constant.

\begin{lemma}
\label{lemma_existance_nu}
For the propagation term $\times(\nabla\phi_{i,1}(\bm\xi_i), \cdots$, $ \nabla\phi_{i,n}(\bm\xi_i), \bm\nu_{i}^{[1]}, \cdots, \bm\nu_{i}^{[m-1]})$ in \eqref{ith_GVF}, there always exists a feasible solution of $m-1$ vectors $\bm\nu_{i}^{[1]}, \cdots, \bm\nu_{i}^{[m-1]}$ satisfying
\begin{align}
\label{feasible_candidate}
\bm\nu_{i}^{[1]}=&[\overbrace{0, \cdots, 0}^n, \overbrace{-1, 1, 0 , \cdots, 0}^{m}]\t\in\mathbb{R}^{n+m},\nonumber\\
\bm\nu_{i}^{[2]}=&[\overbrace{0, \cdots, 0}^n, \overbrace{-1, 0,1 , \cdots, 0}^{m}]\t\in\mathbb{R}^{n+m},\nonumber\\
\vdots~~=&~~~~~~~~~~~~~~\vdots\nonumber\\
\bm\nu_{i}^{[m-1]}=&[\overbrace{0, \cdots, 0}^n, \overbrace{-1, 0,0 , \cdots, 1}^{m}]\t\in\mathbb{R}^{n+m},
\end{align}
such that 
\begin{tikzpicture}[overlay, remember picture]
     \node[anchor=west] at (3.8, -2.0) {\rotatebox{180}{$\left\{ \begin{array}{c} \\ \\ \\  \\ \end{array} \right.$}};
     \node[anchor=west] at (3.8, -3.4) {\rotatebox{180}{$\left\{ \begin{array}{c} \\ \\   \end{array} \right.$}};
     \node[anchor=west] at (4.55, -2.0) {$\scriptstyle n$};
     \node[anchor=west] at (4.55, -3.4) {$\scriptstyle m$};
\end{tikzpicture}
\begin{align}
\label{generalized_split_form}
&\times(\nabla\phi_{i,1}(\bm\xi_i), \cdots,  \nabla\phi_{i,n}(\bm\xi_i), \bm\nu_{i}^{[1]}, \cdots, \bm\nu_{i}^{[m-1]})\nonumber\\
&=
{\scriptsize
\begin{bmatrix}
(-1)^n\big(\partial f_{i,1}^{[1]}+\partial f_{i,1}^{[2]}+\cdots+\partial f_{i,1}^{[m]}\big)\\
(-1)^n\big(\partial f_{i,2}^{[1]}+\partial f_{i,2}^{[2]}+\cdots+\partial f_{i,2}^{[m]}\big)\\
\vdots\\
(-1)^n\big(\partial f_{i,n}^{[1]}+\partial f_{i,n}^{[2]}+\cdots+\partial f_{i,n}^{[m]}\big)\\
\hline
(-1)^n\\
\vdots\\
(-1)^n
\end{bmatrix}}~~~~\in\mathbb{R}^{n+m}, 
\end{align}
where the terms $\nabla\phi_{i,j}(\bm\xi_i), \partial f_{i,j}^{[k]}, j\in\mathbb{Z}_1^n, k\in\mathbb{Z}_1^m$ are given in \eqref{gradient_phi_ij} and \eqref{eq_partial_fij}, respectively.
\end{lemma}

\begin{proof}
For the readers' convenience, we divide the proof into two cases. 

{\textbf{Case 1}: $m=1$.}
If Case 1 holds, it is straightforward that there are no additional auxiliary vectors $\bm\nu_{i}^{[1]}, \cdots, \bm\nu_{i}^{[m-1]}$, which implies that the propagation term in \eqref{ith_GVF} degenerates to be $\times(\nabla\phi_{i,1}(\bm\xi_i), $ $\cdots, \nabla\phi_{i,n}(\bm\xi_i))$ with 
\begin{align}
\label{first_orderi_gradient_phi}
\nabla\phi_{i,j}(\bm\xi_i)=[0, \cdots, 1, \cdots, -\partial f_{i,j}^{[1]} ]\t\in\mathbb{R}^{n+1}, j\in\mathbb{Z}_1^n.
\end{align}
Then, it follows from \eqref{first_orderi_gradient_phi} and the generalized cross product~in Definition~\ref{def_cross_product} that
\begin{align}
\label{first_m_propergation}
&\times(\nabla\phi_{i,1}(\bm\xi_i), \cdots, \nabla\phi_{i,n}(\bm\xi_i))\nonumber\\
=&\sum_{j=1}^{n+1}(-1)^{j-1}\mathrm{det}(G_j)\bold{b}_j,
\end{align}
where $\bold{b}_j=[0, \cdots, 1, \cdots, 0]\t\in\mathbb{R}^{n+1}$ represents the basis column vector, and $\mathrm{det}(G_j), j\in\mathbb{Z}_{1}^{n+1}$ represents the determinants of the submatrices $G_j\in\mathbb{R}^{n\times n},$ by deleting $j$-th column of the matrix of $[\nabla\phi_{i,1}(\bm\xi_i), \cdots, \nabla\phi_{i,n}(\bm\xi_i]\t\in\mathbb{R}^{n\times (n+1)}$ in Definition~\ref{def_cross_product}. 
According to the special structure of $[\nabla\phi_{i,1}(\bm\xi_i), \cdots$, $\nabla\phi_{i,n}(\bm\xi_i]\t$, the submatrices $G_j, j\in\mathbb{Z}_1^{n+1}$ in \eqref{first_m_propergation} can be divided into two subgroups of $j\in\mathbb{Z}_1^n$ and $j=n+1$ below,
\begin{align}
\label{G_j_extention}
G_j=&
{\scriptsize
\begin{bmatrix}
1 & 0 & 0 & \cdots & 0 & -\partial f_{i,1}^{[1]}\\
\vdots & \vdots & \vdots & \ddots & \vdots & \vdots\\
\hline
0 & \cdots &  0& \cdots & 0 & -\partial f_{i,j}^{[1]}\\
\hline
\vdots & \vdots & \vdots & \ddots & \vdots & \vdots\\
0 & 0 & 0 & \cdots & 1 & -\partial f_{i,n}^{[1]}\\
\end{bmatrix}}~~~~~, 
G_{n+1}=I_n.
\end{align}
\begin{tikzpicture}[overlay, remember picture]
    \node[anchor=west] at (5.2, 1.45) {$\scriptstyle j$-$\scriptstyle\mathrm{th}$};
\end{tikzpicture}

Using the Laplace expansion \cite{mirsky2012introduction} along $j$-th row of $G_j, j\in\mathbb{Z}_1^n$, one has that the determinants $\mathrm{det}(G_j)$ in \eqref{G_j_extention} becomes
\begin{align}
\label{G_j_extention_2}
\mathrm{det}(G_j)=(-1)^{n+j+1}\partial f_{i,j}^{[1]}, j\in\mathbb{Z}_1^{n},~\mathrm{det}(G_{n+1})=1.
\end{align}
Substituting \eqref{G_j_extention_2} into \eqref{first_m_propergation} yields 
\begin{align}
\label{first_m_propergation_2}
&\times(\nabla\phi_{i,1}(\bm\xi_i), \cdots, \nabla\phi_{i,n}(\bm\xi_i))\nonumber\\
=&\sum_{j=1}^{n}(-1)^{n+2j}\partial f_{i,j}^{[1]}\bold{b}_j+(-1)^{n}\bold{b}_{n+1}\nonumber\\
=&[(-1)^n\partial f_{i,1}^{[1]}, \cdots, (-1)^n\partial f_{i,n}^{[1]}, (-1)^n ]\t\in\mathbb{R}^{n+1}
\end{align}
with $(-1)^{n+2j}=(-1)^{n}, \forall j\in\mathbb{Z}_1^n$ and $\bold{b}_j$ given in \eqref{first_m_propergation}. Therefore, the decomposition of \eqref{first_m_propergation_2} naturally adheres to the generalized form in \eqref{generalized_split_form}.

{\textbf{Case 2}: $m\geq2, m\in\mathbb{Z}^{+}$.}
If Case 2 holds, one has that $m-1$ auxiliary vectors $\bm\nu_{i}^{[1]}, \cdots, \bm\nu_{i}^{[m-1]}$ are required. 
Let $P_i\in\mathbb{R}^{n+m}, i\in\mathcal V,$ denote the $i$-th propagation term in \eqref{ith_GVF} below
\begin{align} 
\label{P_i_propergation}
P_i:=&\times Q_i,
\end{align}
with the $i$-th augmented matrix $Q_i\in\mathbb{R}^{(n+m-1)\times(n+m)}$ being
\begin{align}
\label{P_i_propergation_1}
Q_i=[\nabla\phi_{i,1}(\bm\xi_i), \cdots, \nabla\phi_{i,n}(\bm\xi_i), \bm\nu_{i}^{[1]}, \cdots, \bm\nu_{i}^{[m-1]}]\t
\end{align}
for conciseness. By substituting the feasible solution \eqref{feasible_candidate} into the propagation term $P_i$ in \eqref{P_i_propergation}, it follows from \eqref{gradient_phi_ij} and the generalized cross product in \cite{galbis2012vector} that 
\begin{align}
\label{m_propergation}
P_i=\sum_{j=1}^{n+m}(-1)^{j-1}\mathrm{det}(G_j)\bold{b}_j, 
\end{align}
where $\mathrm{det}(G_j)$ denote 
the determinant of $G_j\in\mathbb{R}^{(n+m-1)\times(n+m-1)}, j\in\mathbb{Z}_1^{n+m}$ by deleting the $j$-th column of the matrix $Q_i$ in \eqref{P_i_propergation_1}.

Moreover, according to the special structure of $Q_i$ in \eqref{P_i_propergation_1}, the submatrices $G_j, j\in\mathbb{Z}_1^{n+m-1},$ in \eqref{m_propergation} can also be divided into two subgroups for convenience, namely, $~j\in\mathbb{Z}_1^n$ and $j\in\mathbb{Z}_{n+1}^{n+m-1}$, respectively.


\textbf{Subgroup (i):} For the subgroup of $~j\in\mathbb{Z}_1^n$, one has that $G_j\in\mathbb{R}^{(n+m-1)\times(n+m-1)}$ in \eqref{m_propergation} becomes
\begin{tikzpicture}[overlay, remember picture]
    \node[anchor=west] at (-1.2, -2.90) { $\scriptstyle j$-$\scriptstyle\mathrm{th}$};
    \node[anchor=west] at (-5.80, -2.05) {$\left\{ \begin{array}{c} \\ \\ \\ \\ \\ \end{array} \right.$};
     \node[anchor=west] at (-5.80, -3.75) {$\Bigg\{$};
     \node[anchor=west] at (-6.10, -3.75) {$\scriptstyle m$};
     \node[anchor=west] at (-6.30, -2.05) {$\scriptstyle n-1$};
     \node[anchor=west] at (-5.4, -0.9) {\rotatebox{270}{$\left\{ \begin{array}{c} \\ \\ \\  \\ \\ \\ \end{array} \right.$}};
     \node[anchor=west] at (-2.6, -0.9) {\rotatebox{270}{$\left\{ \begin{array}{c} \\  \\ \\ \\ \\ \\ \\  \\ \\ \\ \end{array} \right.$}};
     \node[anchor=west] at (-4.4, -0.45) {$\scriptstyle n-1$};
     \node[anchor=west] at (-0.6, -0.45) {$\scriptstyle m$};
\end{tikzpicture}
\begin{align}
\label{former_G_j}
G_j=&\nonumber\\
&
{\scriptsize \begin{bmatrix}
1 & \cdots & 0 & \cdots & 0 & -\partial f_{i,1}^{[1]} & -\partial f_{i,1}^{[2]} & \cdots & -\partial f_{i,1}^{[m]}\\
\vdots & \ddots & \vdots & \ddots & \vdots & \vdots & \vdots & \ddots & \vdots\\
\hline
0 & \cdots &  0& \cdots & 0 & -\partial f_{i,j}^{[1]} & -\partial f_{i,j}^{[2]} & \cdots & -\partial f_{i,j}^{[m]}\\
\hline
\vdots & \ddots & \vdots & \ddots & \vdots & \vdots & \vdots & \ddots & \vdots\\
0 & \cdots & 0 & \cdots & 1 & -\partial f_{i,n}^{[1]} & -\partial f_{i,n}^{[2]} & \cdots & -\partial f_{i,n}^{[m]}\\
0 & \cdots & 0 & \cdots & 0 & -1 & 1 & \cdots & 0\\
\vdots & \ddots & \vdots & \ddots & \vdots & \vdots & \vdots & \ddots & \vdots\\
0 & \cdots & 0 & \cdots & 0 & -1 & 0 & \cdots & 1\\
\end{bmatrix}}.
\quad
\end{align}

Using the Laplace expansion \cite{mirsky2012introduction} along $j$-th row of $G_j$ in~\eqref{former_G_j}, one has that
\begin{align}
\label{former_G_j_2}
\mathrm{det}(G_j)=&\sum_{k=1}^{m}(-1)^{n+j+k}\partial f_{i,j}^{[k]}\mathrm{det}(H_k),
\end{align}
where $\mathrm{det}(H_k)$ is the minor determinant of the $k$-th matrix $H_k\in\mathbb{R}^{(n+m-2)\times(n+m-2)}$ by removing  
the $j$-th row and $(k+n)$-th column of the matrix $G_i$ in \eqref{former_G_j}, which can be formulated to be
\begin{align}
\label{subblock_H}
H_k=\begin{bmatrix}
H_k^{[1]}  &  H_k^{[2]}\\
H_k^{[3]} & H_k^{[4]}
\end{bmatrix}.
\end{align}
Here, the submatrices $H_k^{[1]}\in\mathbb{R}^{(n-1)\times(n-1)}, H_k^{[2]}\in\mathbb{R}^{(n-1)\times(m-1)}, $ $H_k^{[3]}\in\mathbb{R}^{(m-1)\times(n-1)}, H_k^{[4]}\in\mathbb{R}^{(m-1)\times(m-1)}$ are four sub-blocks of the matrix $G_j$ in \eqref{former_G_j} after removing the $j$-th row and $(k+n)$-th column.

For $k=1$, one has that $H_1^{[1]}=I_{n-1}, H_1^{[3]}=\bold{0}_{(m-1)\times (n-1)}$, $ H_1^{[4]}=I_{m-1}$
with $I_{n-1}, I_{m-1}$ and $\bold{0}_{(m-1)\times (n-1)}$ being the $(n-1)$-, $(m-1)$-dimensional identity matrices and $(m-1)\times(n-1)$ zero matrix, respectively, which follows from \eqref{subblock_H} that
\begin{align}
\label{H1_subblock}
H_1=
\begin{bmatrix}
I_{n-1} & H_1^{[2]}\\
\bold{0}_{(m-1)\times(n-1)} & I_{m-1}
\end{bmatrix}.
\end{align}
Accordingly, one has that 
\begin{align}
\label{H1_determinant}
\mathrm{det}(H_1)=\mathrm{det}(I_{n-1})\mathrm{det}(I_{m-1})=1.
\end{align}
For $k\in\mathbb{Z}_2^{m}$, one has that $H_k^{[1]}=I_{n-1}, H_k^{[3]}=\bold{0}_{(m-1)\times (n-1)}$ as well, which then follows from \eqref{subblock_H} that
\begin{align}
\label{Hk_subblock}
H_k=
\begin{bmatrix}
I_{n-1} & H_k^{[2]}\\
\bold{0}_{(m-1)\times(n-1)} & H_k^{[4]}
\end{bmatrix}
\end{align}
with $H_k^{[4]}\in\mathbb{R}^{(m-1)\times (m-1)}$ being
\begin{tikzpicture}[overlay, remember picture]
    \node[anchor=west] at (1.2, -2.05) {$\scriptstyle k-1$-$\scriptstyle\mathrm{th}$};
\end{tikzpicture}
\begin{align}
\label{matrix_D_k}
H_k^{[4]}=
\begin{bmatrix}
-1 & 1 &  0 & \cdots & 0 & \cdots & 0\\
-1 & 0 & 1& \cdots & 0 & \cdots & 0\\
\vdots & \vdots & \vdots & \ddots & \vdots & \vdots & \vdots\\
\hline
-1 & 0 & 0 & \cdots & 0 & \cdots & 0 \\
\hline
\vdots & \vdots & \vdots & \ddots & \vdots & \vdots & \vdots\\
-1 & 0 & 0  & \cdots & 0 & \cdots & 1 \\
\end{bmatrix}~~~~~~~~.
\end{align}
Analogous to the determinant of the block matrices in \eqref{H1_determinant}, it, together with \eqref{Hk_subblock}, gives that
\begin{align}
\label{Hk_determinant}
\mathrm{det}(H_k)=\mathrm{det}(I_{n-1})\mathrm{det}(H_k^{[4]})=\mathrm{det}(H_k^{[4]}).
\end{align}
Since $H_k^{[4]}, \forall k\in\mathbb{Z}_2^m$ contain the same rows but with different orderings, it follows from the swapping property of the determinant \cite{mirsky2012introduction} that 
\begin{align}
\label{swap_property}
\mathrm{det}(H_k^{[4]})=-\mathrm{det}(H_{k-1}^{[4]})=\cdots=(-1)^{k-2}\mathrm{det}(H_{2}^{[4]})
\end{align}
with 
\begin{align}
\label{matrix_D_k_2}
H_2^{[4]}=
\begin{bmatrix}
-1 & 0 &  \cdots & 0\\
-1 & 1 &  \cdots & 0\\
\vdots &  \vdots & \ddots & \vdots \\
-1 & 0  &  \cdots & 1 \\
\end{bmatrix}.
\end{align}
Substituting \eqref{swap_property} and \eqref{matrix_D_k_2} into \eqref{Hk_determinant} yields
\begin{align}
\label{Hk_determinant_1}
\mathrm{det}(H_k)=(-1)^{k-2}\mathrm{det}(H_{2}^{[4]})=(-1)^{k-1}, \forall k\in\mathbb{Z}_2^m.
\end{align}
By formulating $\mathrm{det}(H_1)$ in \eqref{H1_determinant} and $\mathrm{det}(H_k), k\in\mathbb{Z}_2^{m}$ in \eqref{Hk_determinant_1} together, one has that
\begin{align}
\label{Hk_determinant_all}
\mathrm{det}(H_k)=(-1)^{k-1}, \forall k\in\mathbb{Z}_1^m.
\end{align}
By substituting \eqref{Hk_determinant_all} into \eqref{former_G_j_2}, we obtain
\begin{align}
\label{former_G_j_2_first}
\mathrm{det}(G_j)=&\sum_{k=1}^{m}(-1)^{n+j+2k-1}\partial f_{i,j}^{[k]}\nonumber\\
=&\sum_{k=1}^{m}(-1)^{n+j-1}\partial f_{i,j}^{[k]}, \forall j\in\mathbb{Z}_1^n.
\end{align}

\textbf{Subgroup (ii):} For the the subgroup of $~j\in\mathbb{Z}_{n+1}^{n+m}$, 
analogous to the subblocks in \eqref{subblock_H}, the matrix $G_j\in\mathbb{R}^{(n+m-1)\times(n+m-1)}$ in \eqref{m_propergation} becomes
\begin{align}
\label{subblock_G_j}
G_j=\begin{bmatrix}
G_j^{[1]} & G_j^{[2]} \\
G_j^{[3]}  & G_j^{[4]} 
\end{bmatrix}, j\in\mathbb{Z}_{n+1}^{n+m},
\end{align}
where $G_j^{[1]}\in\mathbb{R}^{n\times n}, G_j^{[2]}\in\mathbb{R}^{n\times (m-1)}, G_j^{[3]}\in\mathbb{R}^{(m-1)\times n}$, $G_j^{[4]}\in\mathbb{R}^{(m-1)\times (m-1)}$. 
According to the dimensionalities of these four submatrices, one has that $G_j^{[1]}=I_n, G_j^{[3]}=\bold{0}_{(m-1)\times n}$, 
and $G_j^{[4]}$ is the same structure as $H_k^{[4]}$ in \eqref{matrix_D_k}, which follows from \eqref{subblock_G_j} that
\begin{align}
\label{subblock_G_j_2}
G_j=\begin{bmatrix}
I_{n} & G_j^{[2]} \\
\bold{0}_{(m-1)\times n} & G_j^{[4]} 
\end{bmatrix}, j\in\mathbb{Z}_{n+1}^{n+m}.
\end{align}
Using the determinant of the block matrices in \eqref{Hk_determinant}, it follows from \eqref{subblock_G_j_2} that
\begin{align}
\label{determinant_G_j}
\mathrm{det}(G_j)=\mathrm{det}(I_{n})\mathrm{det}(G_j^{[4]})=\mathrm{det}(G_j^{[4]}).
\end{align}
Since the structure of $G_j^{[4]}$ is the same to $H_k^{[4]}$ when $j=k+n, j\in\mathbb{Z}_{n+1}^{n+m}$, it follows from \eqref{swap_property} that 
\begin{align}
\label{determinant_G_j_2_latter}
\mathrm{det}(G_j)=(-1)^{j-1-n}, j\in\mathbb{Z}_{n+1}^{n+m}.
\end{align}
Substituting $\mathrm{det}(G_j)$ in \eqref{former_G_j_2_first} and \eqref{determinant_G_j_2_latter} of \textbf{Subgroups (i)-(ii)} into the propergation term $P_i$ in \eqref{m_propergation} yields
\begin{align}
\label{new_p_i}
P_i=&\sum_{j=1}^{n}(-1)^{n+2j-2}\sum_{k=1}^{m}\partial f_{i,j}^{[k]}\bold{b}_j + \sum_{j=n+1}^{n+m}(-1)^{2j-2-n}\bold{b}_j\nonumber\\
=&\sum_{j=1}^{n}\sum_{k=1}^{m}(-1)^{n}\partial f_{i,j}^{[k]}\bold{b}_j+ \sum_{k=1}^{m}(-1)^{n}\bold{b}_{n+k},
\end{align}
where $k=j-n,~\mathrm{if}~j\in\mathbb{Z}_{n+1}^{n+m}$ for the second term. From the definition of $\bold{b}_{j}$ in~\eqref{first_m_propergation}, the propagation term $P_i$ in \eqref{P_i_propergation} is formulated to be
\begin{align}
\label{generalized_split_form_2}
P_i=&\bigg[\overbrace{(-1)^n\sum_{k=1}^m\partial f_{i,1}^{[k]}, \cdots, (-1)^n\sum_{k=1}^m\partial f_{i,n}^{[k]}}^n,\nonumber\\
&\underbrace{(-1)^n, \cdots, (-1)^n}_{m} \bigg]\t\in\mathbb{R}^{n+m},
\end{align}
which aligns with \eqref{generalized_split_form}. The proof is thus completed.
\end{proof}

\subsection{CGVF Controller Design}
From Lemma~\ref{lemma_existance_nu}, it follows from~\eqref{condition_target_virtual_coordinate}, \eqref{GVF_dynamic}, \eqref{generalized_split_form}  that the derivatives of the target virtual coordinates $\dot{\bm\omega_{\ast}}:=\bold{u}_{\ast}^{\omega}$ become
\begin{align}
\label{derivative_target_omega}
\dot{\bm\omega}_{\ast}=[ (-1)^n, \cdots, (-1)^n]\t\in\mathbb{R}^{m}.
\end{align} 
Meanwhile, since the convergence term $-\sum_{j=1}^n k_{i,j}\phi_{i,j}(\bm\xi_i)$ $ \nabla\phi_{i,j}(\bm\xi_i)$ in \eqref{ith_GVF} of Definition~\ref{def_GVF_manifold} can be split into 
\begin{align}
\label{split_second_term}
&-\sum_{j=1}^n k_{i,j}\phi_{i,j}(\bm\xi_i) \nabla\phi_{i,j}(\bm\xi_i)\nonumber\\
=&\bigg[\overbrace{-k_{i,1}\phi_{i,1}(\bm\xi_i), \cdots, -k_{i,n}\phi_{i,n}(\bm\xi_i)}^n,\nonumber\\
&\underbrace{\sum_{j=1}^n k_{i,j}\phi_{i,j} \partial f_{i,j}^{[1]}, \cdots, \sum_{j=1}^n k_{i,j}\phi_{i,j} \partial f_{i,j}^{[m]}}_m\bigg]\t\in\mathbb{R}^{n+m}.
\end{align}
By adding \eqref{generalized_split_form_2}, \eqref{split_second_term}, and including extra terms to account for the target attraction and neighboring repulsion in $-c_i(\omega_{i,l}-\omega_{\ast,l})-\sum_{k\in\mathcal N_i(t)}$ and $\alpha(\|\bm\omega_{i}-\bm\omega_{k}\|){(\omega_{i,l}-\omega_{i,k})}/{\|\bm\omega_{i}-\bm\omega_{k}\|}$, respectively,
the CGVF control algorithm $\bold{u}_i, \bold{u}_i^{\omega}$ for robot $i, i\in\mathcal V,$ is redesigned below
\begin{align}
\label{MOFM_navigation}
u_{i,j}=&(-1)^n\sum_{l=1}^{m}\partial f_{i,j}^{[l]}-k_{i,j}\phi_{i,j}, j\in\mathbb{Z}_1^n,\nonumber\\
u_{i,l}^{\omega}=&(-1)^n+\sum_{j=1}^n k_{i,j}\phi_{i,j} \partial f_{i,j}^{[l]}-c_i(\omega_{i,l}-\omega_{\ast,l})\nonumber\\
&+\sum_{k\in\mathcal N_i(t)}\alpha(\|\bm\omega_{i}-\bm\omega_{k}\|)\frac{(\omega_{i,l}-\omega_{i,k})}{\|\bm\omega_{i}-\bm\omega_{k}\|}, l\in\mathbb{Z}_1^m,
\end{align}
where $u_{i,j}\in\mathbb{R}, u_{i,l}^{\omega}\in\mathbb{R}$ are the $j$-th, and $l$-th entries of the inputs $\bold{u}_i$ and $\bold{u}_i^{\omega}$ in \eqref{GVF_dynamic}, respectively, $\omega_{\ast, l}, l\in\mathbb{Z}_1^m$ denotes the $l$-th term of the target virtual coordinates $\bm\omega_{\ast}$
in \eqref{derivative_target_omega}, $k_{i,j}\in\mathbb{R}^{+}$ are the control gains, $\phi_{i,j}$ and $\partial f_{i,j}^{[l]}$ are given~in \eqref{ith_GVF} and \eqref{eq_partial_fij} for conciseness. 
Here, the term $\sum_{k\in\mathcal N_i(t)}\alpha(\|\bm\omega_{i}-\bm\omega_{k}\|)(\omega_{i,l}-\omega_{i,k})/\|\bm\omega_{i}$ $-\bm\omega_{k}\|$ represents a repulsion term among neighboring virtual coordinates with robot $i$, $\mathcal N_i(t)$ is a sensing neighbor set of time $t$ in \eqref{sensing_neighbor}, and $\alpha(s): (r, +\infty)\rightarrow[0, +\infty)$ is a twice continuously differentiable potential function satisfying the following three properties \cite{hu2023cooperative}
\begin{align}
\label{alpha_function}
&i)~\alpha(s)\,\text{decreases}~\mathrm{if}~s\,\text{increases},~\forall s\in(r, R),~\nonumber\\
&ii) \lim_{s\rightarrow r^{+}}\alpha(s)=+\infty,~iii)~\alpha(s)=0, \forall s\in[R, +\infty),
\end{align}
with $r, R$ being the safe and sensing radii in \eqref{condition_ordering_flexible} and $r^{+}$, being the right limit of $r$. 
According to the aforementioned properties i)-iii) in \eqref{alpha_function}, an illustrative example is chosen below \cite{hu2024coordinated}, 
\begin{align}
\label{alpha_example}
&\alpha(s)=\frac{(s-R)^2}{(s-r)^2},~\mathrm{if}~s\in(r, R],~\alpha(s)=0,~\mathrm{if}~s\in(R, +\infty),
\end{align}
where the derivative of $\alpha(s)$ in \eqref{alpha_example} w.r.t $s$ becomes 
\begin{align*}
&\frac{\partial \alpha(s)}{\partial s}=\frac{2(s-R)(R-r)}{(s-r)^3},~\mathrm{if}~s\in(r, R],\nonumber\\
&\frac{\partial \alpha(s)}{\partial s}=0,~\mathrm{if}~s\in(R, +\infty).
\end{align*}
Hence, $({\partial \alpha(s)}/{\partial s})$ is continuous at $R$ as well. The continuity of $\alpha(s), ({\partial \alpha(s)}/{\partial s})$ will be utilized in the convergence analysis in Lemma~\ref{lemma_manifold_convergence} later.

\section{Convergence Analysis}
\label{sec_convergence}
Due to the special design of the repulsion term $\sum_{k\in\mathcal N_i(t)}\alpha(\|\bm\omega_{i}-\bm\omega_{k}\|){(\omega_{i,l}-\omega_{i,k})}/{\|\bm\omega_{i}-\bm\omega_{k}\|}$ in \eqref{MOFM_navigation}, it follows from \eqref{alpha_function} that the proposed CGVF controller~\eqref{MOFM_navigation} is not well-defined under the situations of $\|\bm\omega_{i}-\bm\omega_{k}\|=0$ or $\|\bm\omega_{i}-\bm\omega_{k}\|\leq r$, i.e., 
\begin{align}
\label{not_well_define}
&\|\bm\omega_{i}-\bm\omega_{k}\|\leq r\Rightarrow u_{i,l}^{\omega}=+\infty, \forall l\in\mathbb{Z}_1^m.
\end{align}
This implies that the proposed CGVF controller~\eqref{MOFM_navigation} may fail to work under \eqref{not_well_define}. Before moving to the detailed convergence analysis, the first step is to 
exclude all the ill-defined points of $\|\bm\omega_{i}(t)-\bm\omega_{k}(t)\|\leq r, \forall t>0$, during the evolution process.

\begin{lemma}
\label{lemma_well_define}
Under the assumption of~{\bf A2}, all the robots~$\mathcal V$ governed by the dynamics \eqref{robot_dynamic} and the proposed CGVF controller~\eqref{MOFM_navigation} satisfies
\begin{align}
\label{well_define_condition}
\|\bm\omega_{i}(t)-\bm\omega_{k}(t)\|>r, \forall t>0, i\in\mathcal V, k\in\mathcal N_i,
\end{align}
all along such that the ill-defined points of the virtual coordinates $\|\bm\omega_{i}-\bm\omega_{k}\|\leq r$ in~\eqref{not_well_define} are prevented. Here, $\bm\omega_{i}(t), \bm\omega_{k}(t)$ are the virtual coordinates of robots $i$ and $k$, respectively, $\mathcal N_i$ is the sensing neighboring given in \eqref{sensing_neighbor}.
\end{lemma}

\begin{proof}
Rewriting all the terms of the CGVF controller \eqref{MOFM_navigation} for robot~$i$ into a compact form,
\begin{align}
\label{argumented_matrix}
F_i^{[l]}:=&[\partial f_{i,1}^{[l]}, \cdots, \partial f_{i,n}^{[l]}]\t\in\mathbb{R}^{n}, \forall l\in\mathbb{Z}_1^m,\nonumber\\
\bold{F}_i:=&[F_i^{[1]}, \cdots, F_i^{[m]} ]\in\mathbb{R}^{n\times m},\nonumber\\
K_i:=&\mbox{diag}\{k_{i,1}, \cdots, k_{i,n}\}\t\in\mathbb{R}^{n},\nonumber\\
\bm\delta_{i}:=&[\delta_{i,1}, \cdots, \delta_{i,m}]\t\in\mathbb{R}^{m},
\end{align}
with the $l$-th attraction-repulsion term $\delta_{i,l}, l\in\mathbb{Z}_1^m$ in \eqref{argumented_matrix} being
\begin{align*}
\delta_{i,l}=&-c_i(\omega_{i,l}-\omega_{\ast,l})+\sum_{k\in\mathcal N_i(t)}\alpha(\|\bm\omega_{i}-\bm\omega_{k}\|)\frac{(\omega_{i,l}-\omega_{i,k})}{\|\bm\omega_{i}-\bm\omega_{k}\|},
\end{align*}
one has that the proposed CGVF controller \eqref{MOFM_navigation} and the dynamic matrix $D_i$ in \eqref{matrix_D} become
\begin{align}
\label{compact_GVF}
\begin{bmatrix}
\bold{u}_i\\
\bold{u}_i^{\omega}
\end{bmatrix}=
\begin{bmatrix}
(-1)^n\bold{F}_i \bold{1}_m-K_i\Phi_i\\
(-1)^n\bold{1}_m+(\bold{F}_i)\t K_i\Phi_i +\bm\delta_{i}
\end{bmatrix}, 
D_i=\begin{bmatrix}
I_n & -\bold{F}_i\\
\bold{0}_{m\times n} & I_{m}
\end{bmatrix}
\end{align} 
with the column vector $\bold{1}_m:=[1, \cdots, 1]\t\in\mathbb{R}^m$. 
Let $\widetilde{\bm\omega}_i:=[\widetilde{\omega}_{i,1}, \cdots, \widetilde{\omega}_{i,m}]\t\in\mathbb{R}^m$ be the error vector between the $i$-th virtual coordinates $\bm\omega_i$ and the target virtual coordinates $\bm\omega_{\ast}$ in \eqref{condition_target_virtual_coordinate} below
\begin{align}
\label{error_virtual_coordinates}
\widetilde{\bm\omega}_i=\bm\omega_i-\bm\omega_{\ast}
\end{align}
with each term given $\widetilde{\omega}_{i,l}:=\omega_{i,l}-\omega_{\ast,l}, l\in\mathbb{Z}_1^m$, one has that
\begin{align}
\label{omgea_equation}
&\bm\omega_{i}-\bm\omega_{k}=\bm\omega_{i}-\bm\omega_{\ast}-(\bm\omega_{k}-\bm\omega_{\ast})=\widetilde{\bm\omega}_{i}-\widetilde{\bm\omega}_{k}.
\end{align}
By substituting~\eqref{compact_GVF}, \eqref{error_virtual_coordinates}, and \eqref{omgea_equation} into \eqref{GVF_dynamic}, we obtain
\begin{align}
\label{closed_loop_system}
\begin{bmatrix}
\dot{\Phi}_i\\
\dot{\widetilde{\bm\omega}}_{i}\\
\end{bmatrix}=&
\begin{bmatrix}
-K_i\Big(I_n+\bold{F}_i(\bold{F}_i)\t\Big)\Phi_i-\bold{F}_i\bm\delta_{i}\\
(\bold{F}_i)\t K_i\Phi_i +\bm\delta_{i}\\
\end{bmatrix},
\end{align}
where $\bm\delta_i$ is formulated to be
\begin{align}
\label{new_delat}
\bm\delta_{i}=&-c_i\widetilde{\bm\omega}_{i}+\sum_{k\in\mathcal N_i(t)}\alpha(\|\widetilde{\bm\omega}_{i}-\widetilde{\bm\omega}_{k}\|)\frac{(\widetilde{\bm\omega}_{i}-\widetilde{\bm\omega}_{k})}{\|\widetilde{\bm\omega}_{i}-\widetilde{\bm\omega}_{k}\|}
\end{align}
Next, we will guarantee the well-defined condition in \eqref{well_define_condition} by contradiction. Recalling~{\bf A2}, one has that the initial states of virtual coordinates satisfy 
\begin{align}
\label{initial_p_i}
\|\bm\omega_{i}(0)-\bm\omega_{k}(0)\|>r, \forall t>0, i\in\mathcal V, k\in\mathcal N_i.
\end{align}
Accordingly, we assume that there exists a finite time $T_1>0$ such that the condition \eqref{well_define_condition} is guaranteed for $t\in[0, T_1)$ but not at $t=T_1$. This implies that arbitrary two robots $i, k$ satisfy
\begin{align}
\label{assump_condition}
\exists i,k\in\mathcal V, \|\bm\omega_{i}(T_1)-\bm\omega_{k}(T_1)\|=r~\mathrm{or}~0.
\end{align}
Then, we pick a Lyapunov candidate function 
\begin{align}
\label{V_1}
V(\Phi_i, \bm\omega_i)=&\sum_{i=1}^{N}\Big\{\Phi_i\t K_i\Phi_i+c_i\widetilde{\bm\omega}_{i}\t\widetilde{\bm\omega}_{i}\Big\}\nonumber\\
&+\sum_{i=1}^N\sum_{k\in\mathcal N_i}\int_{\|\widetilde{\bm\omega}_{i}-\widetilde{\bm\omega}_{k}\|}^R\alpha(s)ds,
\end{align}
which is nonnegative and differentiable in the time interval of $t\in[0, T_1)$.
Then, for $t\in[0, T_1)$, taking the derivative of $V(\Phi_i, \bm\omega_i)$ along the time $t$ is 
\begin{align}
\label{dot_V}
\dot{V}(\Phi_i, \bm\omega_i)=&2\sum_{i=1}^N\Big\{\Phi_i\t K_i\dot{\Phi}_i+c_i\widetilde{\bm\omega}_{i}\t\dot{\widetilde{\bm\omega}}_{i}\Big\}-\sum_{i=1}^N\sum_{k\in\mathcal N_i}\nonumber\\
&\alpha(\|\widetilde{\bm\omega}_{i}-\widetilde{\bm\omega}_{k}\|)\frac{(\widetilde{\bm\omega}_{i}-\widetilde{\bm\omega}_{k})\t}{\|\widetilde{\bm\omega}_{i}-\widetilde{\bm\omega}_{k}\|}(\dot{\widetilde{\bm\omega}}_{i}-\dot{\widetilde{\bm\omega}}_{k}).
\end{align}
From the definition of $\mathcal N_i$ and $\alpha(\cdot)$ in \eqref{sensing_neighbor}  and \eqref{alpha_function}, one has that $\alpha(\|\widetilde{\bm\omega}_{i}-\widetilde{\bm\omega}_{k}\|)=\alpha(\|\widetilde{\bm\omega}_{k}-\widetilde{\bm\omega}_{i}\|), \forall k\in\mathcal N_i,$ and $\alpha(\|\widetilde{\bm\omega}_{i}-\widetilde{\bm\omega}_{k}\|)=0, \forall k\notin \mathcal N_i$, which  implies that 
\begin{align}
\label{equal_alpha_0}
&\sum_{i=1}^N\sum_{k\notin\mathcal N_i}\alpha(\|\widetilde{\bm\omega}_{i}-\widetilde{\bm\omega}_{k}\|)\frac{(\widetilde{\bm\omega}_{i}-\widetilde{\bm\omega}_{k})\t}{\|\widetilde{\bm\omega}_{i}-\widetilde{\bm\omega}_{k}\|}(\dot{\widetilde{\bm\omega}}_{i}-\dot{\widetilde{\bm\omega}}_{k})=0,\nonumber\\
&\sum_{i=1}^N\sum_{k=1}^N\alpha(\|\widetilde{\bm\omega}_{i}-\widetilde{\bm\omega}_{k}\|)\frac{(\widetilde{\bm\omega}_{i}-\widetilde{\bm\omega}_{k})\t}{\|\widetilde{\bm\omega}_{i}-\widetilde{\bm\omega}_{k}\|}\dot{\widetilde{\bm\omega}}_{i}\nonumber\\
=&\sum_{i=1}^N\sum_{k=1}^N\alpha(\|\widetilde{\bm\omega}_{k}-\widetilde{\bm\omega}_{i}\|)\frac{(\widetilde{\bm\omega}_{k}-\widetilde{\bm\omega}_{i})\t}{\|\widetilde{\bm\omega}_{k}-\widetilde{\bm\omega}_{i}\|}\dot{\widetilde{\bm\omega}}_{k}.
\end{align}
Since $k\in\mathcal N_i\cup k\notin\mathcal N_1\Leftrightarrow k\in\mathcal V$ (i.e., $k\in\mathbb{Z}_1^N$), it follows from \eqref{dot_V} and \eqref{equal_alpha_0} that
\begin{align}
\label{equal_alpha}
&\sum_{i=1}^N\sum_{k\in\mathcal N_i}\alpha(\|\widetilde{\bm\omega}_{i}-\widetilde{\bm\omega}_{k}\|)\frac{(\widetilde{\bm\omega}_{i}-\widetilde{\bm\omega}_{k})\t}{\|\widetilde{\bm\omega}_{i}-\widetilde{\bm\omega}_{k}\|}(\dot{\widetilde{\bm\omega}}_{i}-\dot{\widetilde{\bm\omega}}_{k})\nonumber\\
=&2\sum_{i=1}^N\sum_{k\in\mathcal N_i}\alpha(\|\widetilde{\bm\omega}_{i}-\widetilde{\bm\omega}_{k}\|)\frac{(\widetilde{\bm\omega}_{i}-\widetilde{\bm\omega}_{k})\t}{\|\widetilde{\bm\omega}_{i}-\widetilde{\bm\omega}_{k}\|}\dot{\widetilde{\bm\omega}}_{i}.
\end{align}
By substituting \eqref{closed_loop_system}, \eqref{new_delat} and \eqref{equal_alpha} into \eqref{dot_V}, we obtain 
\begin{align}
\label{dot_V2}
\dot{V}(\Phi_i, \bm\omega_i)=&2\sum_{i=1}^N\Big\{\Phi_i\t K_i\Big(-K_i\Big(I_n+\bold{F}_i(\bold{F}_i)\t\Big)\Phi_i-\bold{F}_i\bm\delta_{i}\Big)\nonumber\\
&-\bm\delta_i\t\Big((\bold{F}_i)\t K_i\Phi_i +\bm\delta_{i}\Big)\Big\}.
\end{align}
Meanwhile, it follows from \eqref{argumented_matrix} that $\Phi_i\t K_i\bold{F}_i\bm\delta_{i}=\bm\delta_i\t(\bold{F}_i)\t$ $ K_i\Phi_i$ is a scalar, which implies that $\dot{V}(\Phi_i, \bm\omega_i)$ in \eqref{dot_V2} becomes
\begin{align}
\label{dot_V3}
\dot{V}(\Phi_i, \bm\omega_i)=&-2\sum_{i=1}^N\Big\{\Phi_i\t \Xi_i\Phi_i+\bm\delta_i\t\bm\delta_{i}\Big\}
\end{align}
with $\Xi_i:=K_i\Big(I_n+\bold{F}_i(\bold{F}_i)\t\Big)K_i\t\in\mathbb{R}^{n\times n}$. Based on the definition of $K_i, \bold{F}_i$ in \eqref{argumented_matrix}, one has that $\Xi_i$ is a symmetry and positive-definite matrix, which implies that $\lambda_{\min}(\Xi_i)>0$.
Using the Courant-Fischer Theorem \cite{parlett1998symmetric}, one has that 
\begin{align*}
\Phi_i\t \Xi_i\Phi_i\geq \lambda_{\min}(\Xi_i)\|\Phi_i\|^2.
\end{align*}
This, together with \eqref{dot_V3}, gives that
\begin{align}
\label{dot_V4}
\dot{V}(\Phi_i, \bm\omega_i)\leq&-2\sum_{i=1}^N\Big\{\lambda_{\min}(\Xi_i)\|\Phi_i\|^2+\|\bm\delta_i\|^2\Big\}\leq 0.
\end{align}
By integrating both sides of \eqref{dot_V4} until $t=T_1$, we obtain
\begin{align}
\label{dot_V5}
&V(T_1)-V(0)\nonumber\\
\leq&-\int_0^{T_1}2\sum_{i=1}^N\Big\{\lambda_{\min}(\Xi_i(s))\|\Phi_i\|^2+\|\bm\delta_i(s)\|^2\Big\}ds.
\end{align}
On one hand, it follows from \eqref{initial_p_i} and \eqref{V_1} that the initial state $V(0)$ is bounded. Combining with the fact that the term $-\int_0^{T_1}2\sum_{i=1}^N\{\lambda_{\min}(\Xi_i(s))\|\Phi_i\|^2+\|\bm\delta_i(s)\|^2\}ds$ is upper bounded, it follows from \eqref{V_1} that $V(T_1)$ is bounded as well. On the other hand, recalling the condition in \eqref{assump_condition}, one has that 
$V(\Phi_i(T_1), \bm\omega_i(T_1))=+\infty$ if  $\|\bm\omega_{i}(T_1)-\bm\omega_{k}(T_1)\|=r~\mathrm{or}~0$, which contradicts \eqref{dot_V5}. Therefore, there does not exist such a finite time $T_1$, and $\|\bm\omega_{i}(t)-\bm\omega_{k}(t)\|>r, \forall t>0, i\in\mathcal V, k\in\mathcal N_i,$ is satisfied all along. The proof is completed.
\end{proof}

Since the proposed GVF controller~\eqref{MOFM_navigation} is guaranteed to be well-defined all along in Lemma~\ref{lemma_well_define}, we are ready to prove the three conditions {\bf C1-C3} in Definition~\ref{def_MOFM_navigation}.

\begin{lemma}
\label{lemma_manifold_convergence}
Under the assumptions of~{\bf A1, A3}, all the robots~$\mathcal V$ governed by the dynamics \eqref{robot_dynamic} and the proposed CGVF controller~\eqref{MOFM_navigation} converge to the desired common manifold $\mathcal M_i^{phy}$, i.e., {\bf C1:} $\lim_{t\rightarrow\infty}\phi_{i,j}(\bold{p}_i(t))=0, \forall i\in\mathcal V, j\in\mathbb{Z}_1^n$.
\end{lemma}

\begin{proof}
Recalling $\dot{V}(\Phi_i, \bm\omega_i)$ in \eqref{dot_V4}, it follows from the condition of $\lambda_{\min}(\Xi_i(s))\|\Phi_i\|^2\geq 0, \|\bm\delta_i(s)\|^2\geq 0$ that $\dot{V}(\Phi_i, \bm\omega_i)=0$ only if $\|\Phi_i\|=0, \|\bm\delta_i\|=0, \forall i\in\mathcal V$, which implies that the large set $\dot{V}(\Phi_i, \bm\omega_i)=0$ only contains one unique solution $\{\Phi_i=\bold{0}_n, \bm\delta_i=\bold{0}_m, \forall i\in\mathcal V\}$, and hence, the invariance set is compact (bounded closed set). Additinally, since $\dot{V}_1(\Phi_i, \bm\omega_i)$ is non-positive in \eqref{dot_V4}, it follows from LaSalle’s invariance principle \cite{khalil2002nonlinear} that the evolutions of $\Phi_i, \bm\delta_i$
will converge to zeros, i.e.,
\begin{align}
\label{phi_delta_convergence}
\lim_{t\rightarrow\infty}\Phi_i(t)=\bold{0}_n, \lim_{t\rightarrow\infty}\bm\delta_i(t)=\bold{0}_m.
\end{align}
Together with the definition $\Phi_i$ in \eqref{GVF_dynamic}, it follows from \eqref{phi_delta_convergence} that $\lim_{t\rightarrow\infty}\phi_{i,j}(\bold{p}_i(t))=0, \forall i\in\mathcal V, j\in\mathbb{Z}_1^n$. The proof is thus completed.
\end{proof}

\begin{lemma}
\label{lemma_manifold_maneuvering}
Under the assumptions of~{\bf A1, A3}, all the robots~$\mathcal V$ governed by the dynamics \eqref{robot_dynamic} and the proposed CGVF controller~\eqref{MOFM_navigation} achieve the on-manifold maneuvering, i.e., {\bf C2:} $\lim_{t\rightarrow\infty}\dot{\bm\omega}_{i}(t)=\lim_{t\rightarrow\infty}\dot{\bm\omega}_{k}(t)\neq\bold{0}_m, \forall i\neq k\in\mathcal V$.
\end{lemma}

\begin{proof}
Recalling the derivative $\dot{\widetilde{\bm\omega}}_{i}$ in \eqref{closed_loop_system} below,
\begin{align}
\label{omega_closed_loop}
\dot{\widetilde{\bm\omega}}_{i}=(\bold{F}_i)\t K_i\Phi_i +\bm\delta_{i},
\end{align}
it follows from \eqref{phi_delta_convergence} and \eqref{omega_closed_loop} that $\lim_{t\rightarrow\infty}\dot{\widetilde{\bm\omega}}_{i}=\bold{0}_m, \forall i\in\mathcal V$. Since $\widetilde{\bm\omega}_i=\bm\omega_i-\bm\omega_{\ast}$ in \eqref{error_virtual_coordinates}, one has that $\lim_{t\rightarrow\infty}\dot{\bm\omega}_{i}(t)=\lim_{t\rightarrow\infty}\dot{\bm\omega}_{k}(t)\neq\bold{0}_m, \forall i\neq k\in\mathcal V$. The proof is thus completed.
\end{proof}

\begin{lemma}
\label{lemma_manifold_OF_coordination}
Under the assumption of~{\bf A4}, all the robots $\mathcal V$ governed by the dynamics \eqref{robot_dynamic} and the proposed CGVF controller~\eqref{MOFM_navigation} achieve the ordering-flexible pattern on the manifold, i.e., {\bf C3:} (a) $\lim_{t\rightarrow\infty}$ ${1}/{N}\sum_{i=1}^N\bm\omega_i(t)-\bm\omega_{\ast}(t)=\bold{0}_m$, (b) $r<\lim_{t\rightarrow\infty}\|\bm\omega_{i}(t)-\bm\omega_{k}(t)\|$ $<R, i\in\mathcal V, \forall k\in\mathcal N_i$.
\end{lemma}

\begin{proof}
We will prove the condition (a) first. Given the condition $\lim_{t\rightarrow\infty}\bm\delta_i(t)=\bold{0}_m, \forall i\in\mathcal V$ in \eqref{phi_delta_convergence}, one has that 
\begin{align}
\label{sum_delta_0}
\lim_{t\rightarrow\infty}\sum_{i=1}^N\bm\delta_i(t)=\bold{0}_m.
\end{align}
Meanwhile, recalling from the definition of $\bm\delta_i$ given in \eqref{new_delat}, it follows from \eqref{sum_delta_0} that
\begin{align}
\label{all_sum_zero}
\lim_{t\rightarrow\infty}\bigg\{&-\sum_{i=1}^Nc_i\widetilde{\bm\omega}_{i}(t)+\sum_{i=1}^N\sum_{k\in\mathcal N_i(t)}\alpha(\|\widetilde{\bm\omega}_{i}(t)-\widetilde{\bm\omega}_{k}(t)\|)\nonumber\\
&\frac{(\widetilde{\bm\omega}_{i}(t)-\widetilde{\bm\omega}_{k}(t))}{\|\widetilde{\bm\omega}_{i}(t)-\widetilde{\bm\omega}_{k}(t)\|}\bigg\}=\bold{0}_m.
\end{align}
Since the term $\sum_{i=1}^N\sum_{k\in\mathcal N_i(t)}\alpha(\|\widetilde{\bm\omega}_{i}-\widetilde{\bm\omega}_{k}\|)\frac{(\widetilde{\bm\omega}_{i}-\widetilde{\bm\omega}_{k})}{\|\widetilde{\bm\omega}_{i}-\widetilde{\bm\omega}_{k}\|}=\bold{0}_m$, one has that \eqref{all_sum_zero} becomes
\begin{align}
\label{partial_sum_omega}
\lim_{t\rightarrow\infty}-\sum_{i=1}^Nc_i\widetilde{\bm\omega}_{i}(t)=\lim_{t\rightarrow\infty}\sum_{i=1}^N\{\bm\omega_{i}(t)-\bm\omega_{\ast}(t)\}=\bold{0}_m.
\end{align}
Together with $\widetilde{\bm\omega}_{i}$ in \eqref{error_virtual_coordinates}, it follows from \eqref{partial_sum_omega} that the condition (a) is formulated to be
\begin{align*}
\lim_{t\rightarrow\infty}\frac{1}{N}\sum_{i=1}^N\bm\omega_i(t)-\bm\omega_{\ast}(t)=\lim_{t\rightarrow\infty}\frac{\sum_{i=1}^N(\bm\omega_i(t)-\bm\omega_{\ast}(t))}{N}=\bold{0}_{m}.
\end{align*}
The condition (a) is guaranteed. Then, we will discuss two inequalities in condition (b). Precisely, recalling \eqref{well_define_condition} in Lemma~\ref{lemma_well_define}, one has that the left-side inequality is satisfied, namely, $r<\lim_{t\rightarrow\infty}\|\bm\omega_{i}(t)-\bm\omega_{k}(t)\|, i\in\mathcal V, \forall k\in\mathcal N_i$. Moreover, since $\|\bm\omega_i-\bm\omega_k\|\leq R$ for any robot $k\in\mathcal N_i$ in~\eqref{sensing_neighbor}, one has that the right-side inequality of the condition (b) is naturally satisfied as well. Accordingly, the two conditions (a)-(b) are both guaranteed, and the proof is completed.
\end{proof}

\begin{theorem}
\label{theorem_CGVF}
Under the assumptions of~{\bf A1-A4}, all the robots~$\mathcal V$ governed by the dynamics \eqref{robot_dynamic} and the proposed CGVF controller~\eqref{MOFM_navigation} achieve the {\it MOFM-Nav}, i.e., {\bf C1-C3} in Definition~\ref{def_MOFM_navigation}. 
\end{theorem}

\begin{proof}
We draw the conclusion from Lemmas \ref{lemma_existance_nu}-\ref{lemma_manifold_OF_coordination} directly.
\end{proof}

\section{Algorithm Verification}
\label{section_algorithm}
In this section, we consider $N=7$ robots described by~\eqref{robot_dynamic} and choose all the parameters to be the same for all simulations. Precisely, the sensing and safe radii $R$ and $r$ for the virtual coordinates are set to be $R=1.6, r=0.4$, respectively. The control gains $k_{i,j}, c_i$ in \eqref{MOFM_navigation} are set to be $k_{i,j}=0.7, c_i=20, \forall i\in\mathcal V, j\in\mathbb{Z}_1^n$. The potential function $\alpha$ is designed in \eqref{alpha_example} according to \cite{hu2024coordinated}.

\begin{figure}[!htb]
\centering
\includegraphics[width=\hsize]{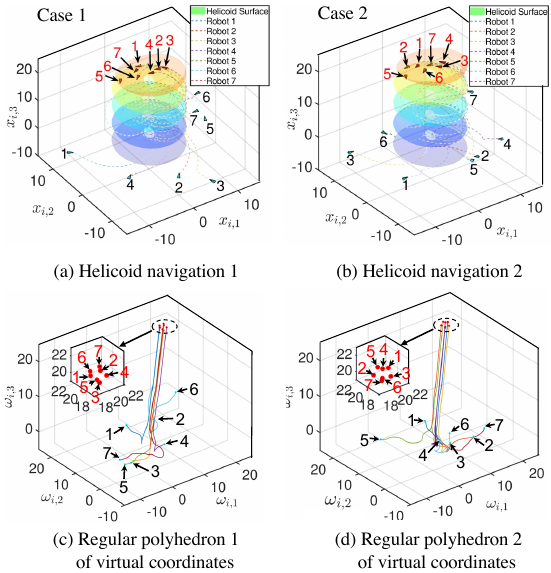}
\caption{\textbf{First Simulation: Flexibility of the {\it MOFM-Nav} on the Helicoid Manifold with Different Initial Positions in Cases 1-2.} (a)-(b) Trajectories of the $7$ robots controlled by the CGVF controller \eqref{MOFM_navigation} under different initial positions. (c)-(d) Trajectories of the corresponding virtual coordinates $\omega_{i,j}$, where $i \in \mathcal{V}$ and $j = 1, 2, 3$. In (a)-(b), the blue and red triangles represent the initial and final positions of the robots, respectively, while the dashed lines indicate their trajectories. Similarly, in (c)-(d), the blue and red circles denote the initial and final positions of the virtual coordinates $\bm{\omega}_i$ for $i \in \mathcal{V}$, with the lines showing their trajectories.}
\label{helicold_surface}
\end{figure}

\begin{figure}[!htb]
\centering
\includegraphics[width=8cm]{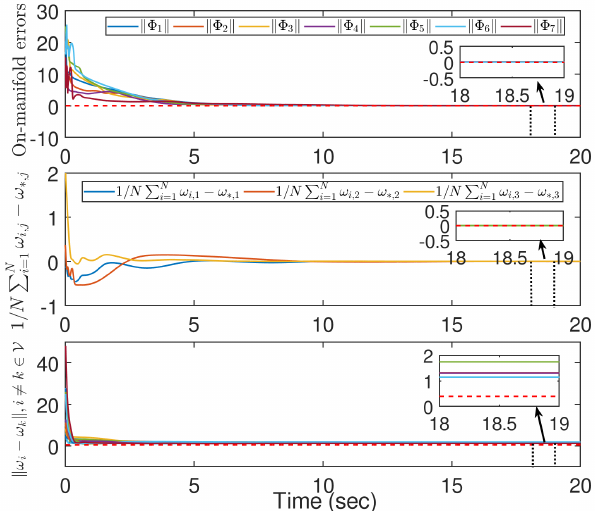}
\caption{Temporal evolution of the on-manifold convergence errors $\|\Phi_i\|$, $i=1, \cdots, 7$, the condition of ordering-flexible coordination ${1}/{N}\sum_{i=1}^N\omega_{i,j}(t)-\omega_{\ast,j}, j=1, 2, 3$ in Definition~\ref{def_MOFM_navigation}, and the relative distance between arbitrary two virtual coordinates $\|\bm\omega_i-\bm\omega_k\|, \forall i\neq k\in\mathcal V$ in Case 1 of Fig.~\ref{helicold_surface} (a), (c) for example. }
\label{spin_state_1}
\end{figure}

\begin{figure}[!htb]
\centering
\includegraphics[width=8cm]{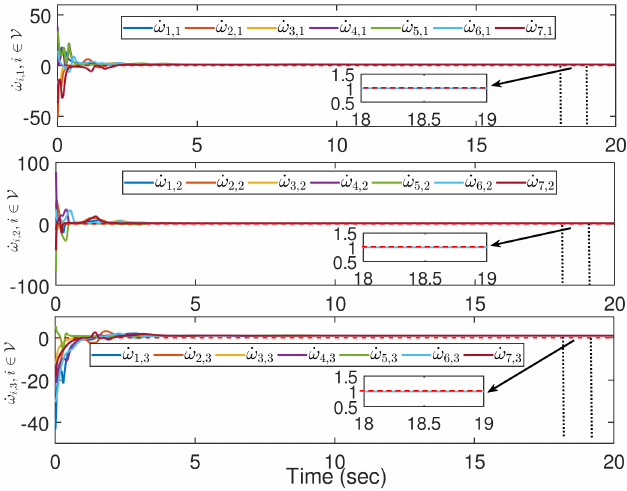}
\caption{Temporal evolution of the derivative of the virtual coordinates $\dot{\omega}_{i,j}, i\in\mathcal V, j=1,2, 3$, in Case 1 of Fig.~\ref{helicold_surface} (a), (c) for example.}
\label{spin_state_2}
\end{figure}

In the first simulation (i.e., Figs.~\ref{helicold_surface}-\ref{spin_state_2}), to show the flexibility of the {\it MOFM-Nav} with different initial positions, we consider a desired Helicoid Manifold with three virtual coordinates. Therein, the parameterization is characterized by
\begin{align*}
x_{i,1} =& (4 + 3\cos\omega_{i,1})\cos\omega_{i,2},\nonumber\\
x_{i,2} =& (4 + 3\cos\omega_{i,1})\sin\omega_{i,2},\nonumber\\
x_{i,3} =& \omega_{i,1} + \sin(\omega_{i,2}+\omega_{i,3}), i\in\mathcal V,
\end{align*}
satisfying the assumptions of {\bf A1, A3, A4}. Figs.~\ref{helicold_surface} (a)-(b) illustrate two cases of seven robots from different initial positions (i.e., blue triangles with black numbers) to the helicoid-manifold navigation with different ordering sequences (i.e., red triangles with red numbers), which verify the flexibility of the proposed CGVF controller \eqref{MOFM_navigation}. Moreover, Figs.~\ref{helicold_surface} (c)-(d) also illustrate the trajectories of the virtual coordinates $\bm\omega_i$ in the 3-D virtual space, where the blue circles (i.e., the initial positions of virtual coordinates satisfying {\bf A2})
to the regular polyhedron of the zoomed-in red circles (i.e., the final positions of virtual coordinates) with different ordering sequences.

\begin{figure}[!htb]
\centering
\includegraphics[width=\hsize]{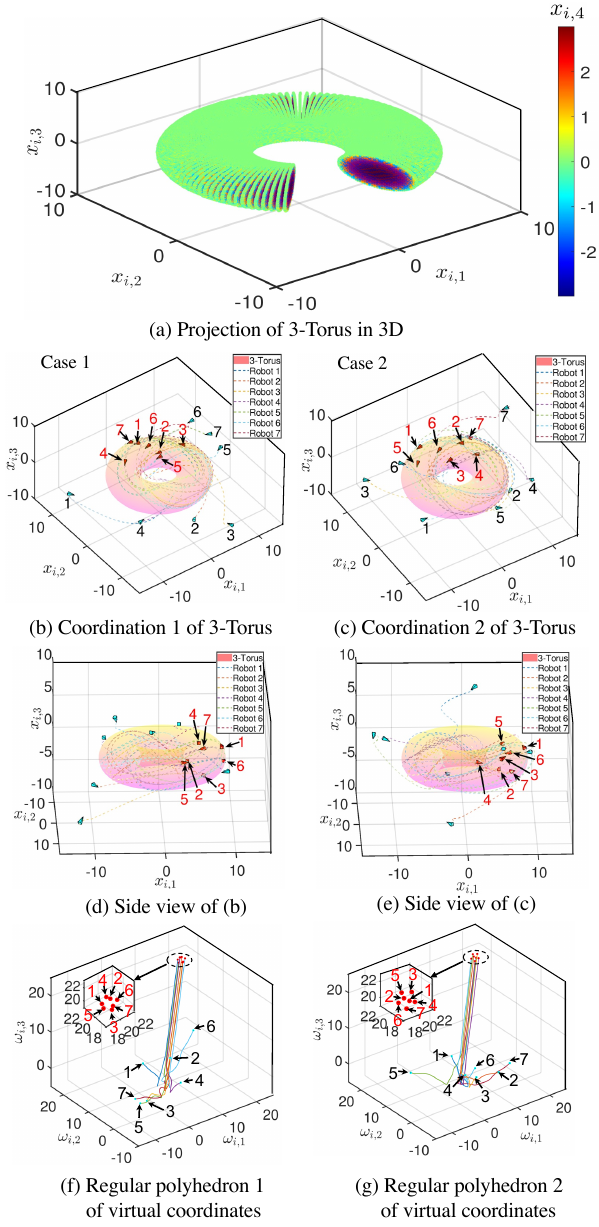}
\caption{\textbf{Second Simulation: Adaptability to the higher-dimensional 3-Torus Manifold in Cases 1-2.} (a) The $3$-D projection of the 3-torus manifold in \eqref{3_torus_equa} with the fourth coordinate represented by varying colors. 
(b)-(c) The $3$-D projection of the first three coordinates of the $7$ robots, governed by the CGVF controller \eqref{MOFM_navigation}, under different initial positions in Cases 1-2. (d)-(e) Side views of the subfigures (b)-(c). (f)-(g) Trajectories of the corresponding three virtual coordinates $\omega_{i,j}$, where $i \in \mathcal{V}$ and $j = 1, 2, 3$ in Cases 1-2. Here, the blue and red triangles, blue and red circles, and dashed colored lines have the same meanings as those in Fig.~\ref{helicold_surface}.
}
\label{3_torus}
\end{figure}

\begin{figure}[!htb]
\centering
\includegraphics[width=\hsize]{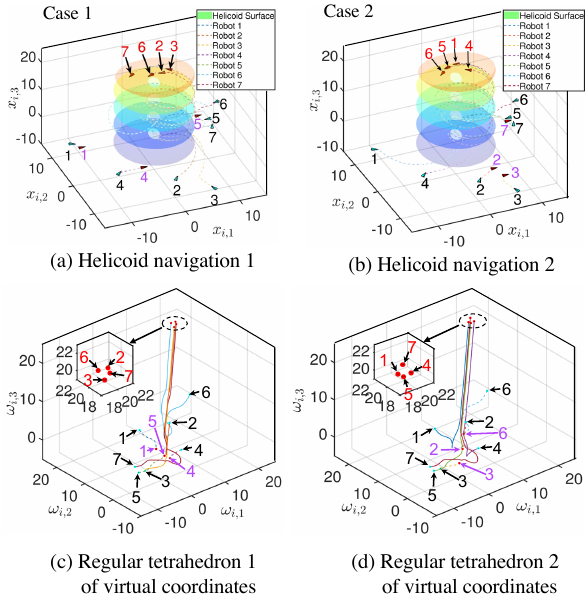}
\caption{\textbf{Third Simulation: Robustness to Some Robots Breakdown in Cases 1-2.} (a)-(b) Trajectories of the $7$ robots when robots $i = 1, 4, 5$ and $i = 2, 3, 6$ suddenly breakdown, respectively. (c)-(b) Trajectories of the corresponding virtual coordinates $\omega_{i,j}, i \in \mathcal{V}$ and $j = 1, 2, 3$ under the scenariso of robots breakdown. Here, the blue and red triangles, blue and red circles, and dashed colored lines have the same meanings as those in Fig.~\ref{helicold_surface}. The purple numbers represent the breakdown robots. }
\label{Broken_down}
\end{figure}

\begin{figure}[!htb]
\centering
\includegraphics[width=\hsize]{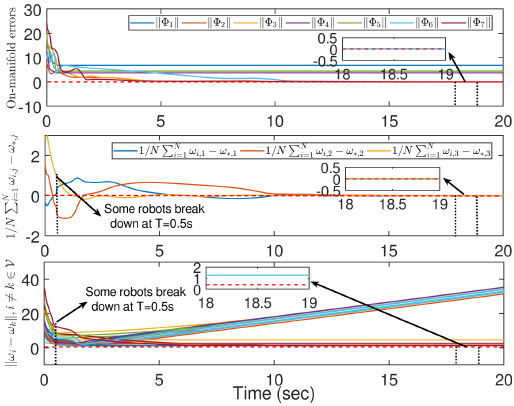}
\caption{Temporal evolution of the on-manifold convergence errors $\|\Phi_i\|$, $i=1, \cdots, 7$, the condition of ordering-flexible coordination ${1}/{N}\sum_{i=1}^N\omega_{i,j}(t)-\omega_{\ast,j}, j=1, 2, 3$ in Definition~\ref{def_MOFM_navigation}, and the relative distance between arbitrary two virtual coordinates $\|\bm\omega_i-\bm\omega_k\|, \forall i\neq k\in\mathcal V$ in Case 1 of Fig.~\ref{Broken_down} (a), (c) for example. }
\label{broken_state_1}
\end{figure}

\begin{figure}[!htb]
\centering
\includegraphics[width=8cm]{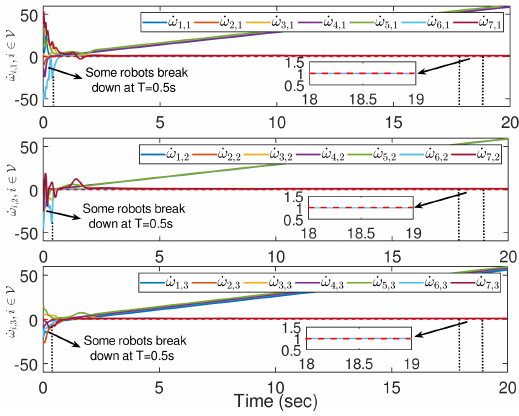}
\caption{Temporal evolution of the derivative of the virtual coordinates $\dot{\omega}_{i,j}, i\in\mathcal V, j=1,2, 3$, in Case 1 of Fig.~\ref{Broken_down} (a), (c) for example. }
\label{broken_state_2}
\end{figure}

Additionally, to quantitatively analyze the effectiveness of the proposed CGVF controller \eqref{MOFM_navigation}, we take Case 1 in Fig.~\ref{helicold_surface}~(a) as an illustrative example. It is observed in Fig.~\ref{spin_state_1} that the on-manifold convergence errors converge to zeros, i.e., $\lim_{t\rightarrow\infty}\|\Phi_i(t)\|=0, i\in\mathbb{Z}_1^7$, which verifies the on-manifold convergence in condition {\bf C1} of Definition~\ref{def_MOFM_navigation}.
Moreover, the evolutions of $\lim_{t\rightarrow\infty}{1}/{N}\sum_{i=1}^N\omega_{i,j}(t)-\omega_{\ast,j}=0, j=1, 2, 3$, and $0.4<\lim_{t\rightarrow\infty}\|\bm\omega_{i}(t)-\bm\omega_{k}(t)\|<1.6, i\in\mathcal V, \forall k\in\mathcal N_i$ in Fig.~\ref{spin_state_1} indicates that the ordering-flexible coordination in condition {\bf C3} of Definition~\ref{def_MOFM_navigation} is achieved. As shown in Fig.~\ref{spin_state_2}, the derivative of the virtual coordinates also converges to zeros, i.e., $\lim_{t\rightarrow\infty}\dot{\omega}_{i,j}(t)=0, j\in\mathbb{Z}_1^3$, which demonstrates the on-manifold maneuvering in condition {\bf C2} of Definition~\ref{def_MOFM_navigation}. Therefore, the definition of the {\it MOFM-Nav} is rigorously guaranteed.

In the second simulation (i.e., Fig.~\ref{3_torus}), to show the adaptability of the CGVF controller \eqref{MOFM_navigation} in a higher-dimensional Euclidean space, we consider a desired 3-torus manifold  in 4-D, whose  parameterization is characterized by
\begin{align}
\label{3_torus_equa}
x_{i,1} =& (6 + 3\cos\omega_{i,1})\cos\omega_{i,2},\nonumber\\
x_{i,2} =& (6 + 3\cos\omega_{i,1})\sin\omega_{i,2},\nonumber\\
x_{i,3} =& 3 \sin\omega_{i,2}\cos\omega_{i,3}, \nonumber\\
x_{i,4} =&3 \sin\omega_{i,2}\sin\omega_{i,3}, i\in\mathcal V,
\end{align}
satisfying the assumptions of {\bf A1, A3, A4}. To better understand the 3-torus manifold in \eqref{3_torus_equa} vividly, Fig.~\ref{3_torus} (a) first illustrates the 3-torus manifold embedded in $3$-D by projecting the fourth coordinate 
into the varying colors. As shown in Figs.~\ref{3_torus} (b)-(c), we also show the projection trajectories of seven robots starting from different initial positions (i.e., blue triangles) to the final on-manifold navigation (i.e., red triangles) with different ordering sequences. Moreover, to better show the trajectory evolutions of the robots on the projected 3-torus manifold, we provide the side views of Figs.~\ref{3_torus} (b)-(c) in Figs.~\ref{3_torus} (d)-(e), which indicates that the robots not only maneuver on the surface but also inside of the torus tube as well. Analogous to Figs.~\ref{helicold_surface} (c)-(d), we also illustrate the evolution of the virtual coordinates as well, where the blue circles (i.e., the initial positions of virtual coordinates satisfying {\bf A2})
to the regular polyhedron of the zoomed-in red circles (i.e., the final positions of virtual coordinates) with different ordering sequences. Therefore, the adaptability of the proposed CGVF controller \eqref{MOFM_navigation} in higher-dimensional Euclidean space is thus verified.

In the third simulation (i.e., Figs.~\ref{Broken_down}-\ref{broken_state_2}), to show the robustness of the CGVF controller \eqref{MOFM_navigation} to some robots breakdown, we consider two kinds of breakdown scenarios in the helicoid-manifold navigation missions, i.e., robots $i=1, 4, 5$ and $i=2, 3, 6$ breakdown at $t=0.2s$, where the initial positions of seven robots are set to be the same as the ones in Fig.~\ref{helicold_surface} (a). Precisely, as shown Figs.~\ref{Broken_down} (a)-(b), the remaining four robots (i.e., $i=2,3,6,7$ in Case 1 and $i=1,4,5,7$ in Case~2) still converge to and maneuver along the Helicoid Manifold with different ordering sequences. Moreover, it is observed in Figs.~\ref{Broken_down}~(c)-(d) that the virtual coordinates of the three breakdown robots stop, while the remaining four virtual coordinates of the robots will form the regular tetrahedron (i.e., zoomed-in small figures) with different ordering sequences, which thus verifies the robustness of the propose CGVF controller. Additionally, we also illustrate the state evolution of Fig.~\ref{Broken_down}~(a) in Figs.~\ref{broken_state_1}-\ref{broken_state_2}. Analogous to Figs.~\ref{spin_state_1}-\ref{spin_state_2}, the evolution of the on-manifold convergence errors satisfy $\lim\|\Phi_i(t)\|=0, i=2,3,6,7$, the evolution of the conditions of ordering-flexible coordination satisfy $\lim_{t\rightarrow\infty}{1}/{N}\sum_{i=1}^N\omega_{i,j}(t)-\omega_{\ast,j}=0, i=2,3,6,7, j=1, 2, 3$, and $0.4<\lim_{t\rightarrow\infty}\|\bm\omega_{i}(t)-\bm\omega_{k}(t)\|<1.6, i=2,3,6,7, \forall k\in\mathcal N_i$ and the evolution of the derivative of virtual coordinates satisfy $\lim_{t\rightarrow\infty}\bm\omega_i=0,  i=2,3,6,7,$ for the remaining four robots, which rigorously achieves all the conditions {\bf C1-C3} in Definition~\ref{def_MOFM_navigation}. Therefore, the robustness of the scenario where some robots break down is verified.

\section{Conclusion}
\label{section_conclusion}
In this paper, we have presented the design of the CGVF algorithm to achieve {\it MOFM-Nav} on the desired $m$-D manifold ($m\geq2$). To bridge the gap between ordering-flexible coordination design and the suitable decoupling of the last $m$ entries of the propagation term in the CGVF algorithm, we have obtained a feasible solution of auxiliary vectors. Then, the global convergence and on-manifold ordering-flexible motion coordination have been both guaranteed. Finally, extensive simulations have been conducted to demonstrate the flexibility, adaptability, and robustness of the proposed algorithm.
\appendices

\bibliographystyle{IEEEtran}
\bibliography{IEEEabrv,ref}

\end{document}